\documentclass{article}

\usepackage{graphicx}
\usepackage{amsmath}
\usepackage{amsfonts}
\usepackage{amssymb}
\usepackage{caption}
\usepackage{subcaption}
\usepackage{siunitx}
\usepackage[T1]{fontenc}
\usepackage{cite}
\usepackage{booktabs}
\usepackage[normalem]{ulem}
\usepackage[dvipsnames]{xcolor}
\usepackage{cleveref}

\usepackage{amsthm}

\newtheorem{theorem}{Theorem}[section] 
\newtheorem{lemma}[theorem]{Lemma}    
\newtheorem{corollary}[theorem]{Corollary}
\newtheorem{assumption}{Assumption}

\theoremstyle{definition}
\newtheorem{definition}{Definition}[section]

\theoremstyle{remark}
\newtheorem*{remark}{Remark} % The asterisk (*) makes it unnumbered

\usepackage{mathtools}
\DeclarePairedDelimiter{\norm}{\lVert}{\rVert}

\usepackage{algorithm}
\usepackage{algpseudocode}

\title{Global Convergence of DGM and PINN Algorithms for Solving Nonlinear PDEs}
\author{Justin Sirignano\footnote{Mathematical Institute, University of Oxford}, Konstantinos Spiliopoulos\footnote{Department of Mathematics and Statistics, Boston University}, and Samuel Cohen\footnote{Mathematical Institute, University of Oxford} }
\date{June 2026}

\begin{document}

\maketitle

\begin{abstract}
The Deep Galerkin Method (DGM) and Physics-Informed Neural Networks (PINNs) have become widely used methods for solving partial differential equations (PDEs) in the rapidly growing field of scientific machine learning. In these methods, a neural network is trained to approximate the PDE solution by using (stochastic) gradient descent to minimize the PDE residual of the neural network. Due to the non-convexity of the PDE residual objective function, the trained neural network may, in principle, only converge to a local minimizer of the objective function (which would not be a solution of the PDE). Therefore, there is a longstanding question regarding the mathematical foundations of these algorithms, and it is highly valuable to establish that the trained neural network will converge to the PDE solution. In this paper, we consider a class of semilinear PDEs with nonlinearities in the solution and its first derivative. For this class of PDEs, we prove that neural networks trained with gradient descent to minimize the PDE residual objective function will converge to the PDE solution as the network width and training time $\rightarrow \infty$. 
\end{abstract}

\section{Introduction}

The Deep Galerkin Method (DGM, \cite{DGM}) and Physics-Informed Neural Networks (PINNs, \cite{PINNs}) methods solve Partial Differential Equations (PDEs) using neural networks. Both methods share a common objective function, the PDE residual of the neural network, which is used to train the neural network to approximate the solution of the PDE. Specifically, the neural network parameters are trained with gradient descent or stochastic gradient descent (SGD) to minimize the PDE residual. These methods have become widely used in scientific machine learning across many diverse fields including fluid dynamics, physics, financial mathematics, biology, and optimal control.  

The PDE residual of a neural network is a non-convex objective function (of the trainable parameters) and therefore, in principle, the trained neural network may converge to a \emph{local minimizer} which is not the PDE solution. This would of course be a major limitation to using neural networks as a numerical method for PDE applications. It has been proven that DGM/PINNs trained with gradient descent will globally converge to the solution of linear elliptic PDEs \cite{DGMLinearGlobalConvergence}. The mathematical proof in \cite{DGMLinearGlobalConvergence} cannot however be applied to nonlinear PDEs. In this paper, for a class of nonlinear PDEs, we prove global convergence of DGM/PINNs trained with gradient descent to the PDE solution. The proof requires the development of a new mathematical approach different from \cite{DGMLinearGlobalConvergence} to address the non-trivial technical challenges introduced by the nonlinearity of the PDE. 

We  consider a class of semilinear PDEs
\begin{eqnarray}
\mathcal{A}_u u &=& g, \phantom{......} x \in \Omega, \notag \\
u &=& 0, \phantom{......} x \in \partial \Omega,
\label{PDE0}
\end{eqnarray}
%$a_{ij}(x,u) \frac{\partial^2}{\partial x_i x_j} + b_i(x,u) \frac{\partial}{\partial x_i} + c(x,u)$.
where $\Omega \subset \mathbb{R}^n$, $\mathcal{A}_u$ is a nonlinear operator of the form $\mathcal{A}_u = \mathcal{L} + f(u, u_x)$, $f(v,w): \mathbb{R} \times \mathbb{R}^n \rightarrow \mathbb{R}$, and $\mathcal{L}$ is the uniformly linear elliptic operator
\begin{eqnarray*}
\mathcal{L} u = -\sum_{i,j=1}^n \frac{\partial}{\partial x_i} [ a_{ij}(x) \frac{\partial u}{\partial x_j}]+ \sum_{i=1}^n c_i(x) \frac{\partial u}{\partial x_i} - d(x) u. 
\end{eqnarray*}

We will approximate the solution $u(x)$ to the PDE (\ref{PDE0}) with a neural network
\begin{eqnarray}
Q^N(x; \theta) = \eta(x) N^{- \beta} \sum_{i=1}^N c^i \sigma( w^i \cdot x + b^i),\label{Eq:NeuralNetwork}
\end{eqnarray}
where $\eta(x)$ is a smooth function which vanishes on the boundary $\partial \Omega$ but has positive normal derivatives on the boundary, $N^{-\beta}$ is a normalization factor with $\frac{1}{2} < \beta < 1$, $\sigma(\cdot)$ is a nonlinear activation function, and the parameters which will be trained are $\theta = (c^i, w^i, b^i)_{i=1}^N$. Due to the function $\eta(x)$, the neural network $Q^N(x; \theta)$ automatically satisfies the Dirichlet boundary conditions of the PDE (\ref{PDE0}). $H^i = \sigma( w^i \cdot x + b^i)$ is the $i$-th hidden unit of the neural network $Q^N(x; \theta)$. $N$ is the total number of hidden units and is often referred to as the ``network width". The related assumptions are stated in Assumption \ref{NeuralNetworkAssumptions}.

The neural network $Q^N(x; \theta)$ is trained to minimize the objective function
\begin{eqnarray}
J^N(\theta) = \frac{1}{2} \int_{\Omega} ( \mathcal{A}_{Q^N } Q^N(x; \theta) - g(x) )^2 \mu(dx)
\label{PDEResidualObjectiveFunction0000}
\end{eqnarray}
using clipped gradient descent (see Section \ref{S: PreliminaryCalculations}), where we denote for convenience $\mathcal{A}_{Q^N } Q^N(x; \theta)=[\mathcal{A}_{Q^N } Q^N](x; \theta)$. The objective function (\ref{PDEResidualObjectiveFunction0000}) is the PDE residual of the neural network $Q^N(x; \theta)$. Since the objective function $J^N(\theta)$ is non-convex in the neural network parameters $\theta$, the gradient descent algorithm for training $\theta$ may potentially converge to a sub-optimal local minimizer of $J^N(\theta)$. A local minimizer would of course not be a solution of the PDE (\ref{PDE0}). We will prove that, as the number of hidden units (i.e., the network width) $N \rightarrow \infty$ and the training time $t \rightarrow \infty$, the neural network will converge to a global minimizer which is the solution of the PDE (\ref{PDE0}). 

$\mu(dx)$ in the objective function (\ref{PDEResidualObjectiveFunction0000}) is a probability distribution, which can be sampled from in order to perform SGD-based training. In the original PINNs algorithm, $\mu$ is fixed before training begins, which is suitable for low-dimensional PDEs. The DGM algorithm randomly samples a minibatch from $\mu$ at each stochastic gradient descent step to calculate an unbiased stochastic estimate of $\nabla_{\theta} J^N(\theta) $ for high-dimensional PDEs. The random sampling of spatial points from $\mu$ at each stochastic gradient descent step yields a mesh-free algorithm which addresses the curse-of-dimensionality and is one of the main distinctions between the PINN and DGM algorithms. In this paper, we will focus on the core optimization question of whether a neural network trained with gradient descent to minimize (\ref{PDEResidualObjectiveFunction0000}) will converge to the PDE solution; further extensions to the stochastic gradient descent-variant are left for future research. 

Due to the universal approximation properties of neural networks \cite{hornik1991approximation}, it is known that \emph{there exists} a neural network which can approximate the solution to a given PDE (if the PDE solution is a function in an appropriate Sobolev space). Universal approximation results for PDEs have been proven in \cite{mishra2023estimates}, \cite{de2024error}, \cite{shin2020convergence}, \cite{doumeche2023convergence}, \cite{cohen1}, \cite{jentzen1}, \cite{jentzen2}, and \cite{cohen2}. These results prove that a neural network \emph{exists} which is an arbitrarily-accurate approximation to the PDE solution, which is an important mathematical result. However, the universal approximation theory literature does not study whether a neural network \emph{trained with gradient descent} will converge to the PDE solution. Convergence of neural networks trained with gradient descent to minimize PDE residual objective functions has been studied for \emph{linear} PDEs in \cite{DGMLinearGlobalConvergence}, \cite{lu2022sobolev}, and \cite{luo2020two}.

Analyzing the optimization of (\ref{PDEResidualObjectiveFunction0000}) must address several challenges. First and foremost, the objective function is non-convex, and therefore in principle the trained neural network may converge to a local minimizer which is not the solution of the PDE. In typical supervised learning problems (e.g., least squares regression), as the number of hidden units $N \rightarrow \infty$, the objective function will ``convexify", leading to a limit problem which is convex in the neural network function and allowing one to prove global convergence \cite{JacotGabrielHongler}. In particular, the training dynamics of the limit neural network in \cite{JacotGabrielHongler} will satisfy a linear ODE. However, for semilinear PDEs, the limit neural network's training dynamics (in the limit $N \rightarrow \infty$) satisfy a non-local nonlinear PDE. Therefore, it would \emph{a priori} seem that the neural network, even in the over-parametrized limit of $N \rightarrow \infty$, may converge to a local minimizer.  

The analysis of the limit neural network training evolution has to address several fundamental challenges. The limit neural network satisfies a nonlinear, non-local PDE involving a nonlinear kernel function $S_{Q(t)}(x,y)$; see (\ref{Eq:S_kernel}) in Section \ref{S: PreliminaryCalculations}. $S_{Q(t)}(x,y)$ is a nonlinear function of the limit neural network $Q(t,x)$ and its derivatives $Q_x(t,x)$, where we will define for notational convenience $Q(t) = Q(t,\cdot)$. Due to this dependence, the kernel operator $S_{Q(t)} f = \int S_{Q(t)}(\cdot,y) f(y) dy$ is a function of time $t$ and therefore its eigenvalues and eigenfunctions also change over time. Consequently, it is not clear if we can study a projection of the PDE residual on a fixed eigenfunction basis (the method used for the mathematical proof in the paper \cite{DGMLinearGlobalConvergence}). \cite{DGMLinearGlobalConvergence} only considered linear PDEs and the limit neural network evolution satisfied a linear PDE. Importantly, for linear elliptic PDEs in \cite{DGMLinearGlobalConvergence}, the limit kernel operator is a constant $S(x,y)$ which does not depend upon the solution $Q(t)$, allowing for the decomposition of the limit PDE dynamics via the fixed eigenfunctions of $S$. Projection onto time-dependent eigenfunctions in the nonlinear case is challenging. Furthermore, the positive eigenvalues of $S_{Q(t)}$ do not have a spectral gap (they accumulate at zero) and $S_{Q(t)}$ may also have zero eigenvalues. This makes it challenging to prove convergence as \emph{a priori} it would seem the zero eigenvalues and eigenvalues accumulating to zero would both lead to a local minimizer (and not a global minimizer). Finally, any attempt at analysis requires establishing appropriate uniform-in-time bounds for $Q(t,x)$ and the PDE residual $\mathcal{A}_{Q(t)} Q(t,x) - g$, which is challenging due to their satisfying a non-local nonlinear PDE. 

This paper develops a new mathematical approach which is able to prove global convergence of the DGM/PINNs algorithm for semilinear PDEs. We first study the evolution of the objective function during training. We establish that when training the  parameters $\theta(t)$ with clipped gradient descent, the objective function $J^{N}(\theta(t))$ will converge to a limit objective function $J(t)=J(Q(t))$ as $N \rightarrow \infty$ and the limit neural network's training trajectory satisfies a nonlinear, non-local PDE (see Section \ref{S: PreliminaryCalculations}). We then prove in Theorem \ref{MainTheoremPDEResidualGlobalConvergence} that the limit objective function  must converge to zero (a global minimizer) as the training time $t \rightarrow \infty$. The first step in our convergence analysis is to prove that the limit training dynamics (which satisfy the nonlinear PDE) converge to a fixed point by establishing the monotonicity of the PDE residual evolution, uniform \emph{a priori} bounds on the neural network in $H^2$, and a uniform regularity bound for the RHS of $\frac{dJ}{dt}$; see (\ref{ObjEvolution}). We analyze this fixed point, which we prove is a global minimizer of the objective function. The trained neural network at the fixed point furthermore has zero PDE residual. The final step is to prove that the neural network converges to the solution of the fixed point equation as training time $t \rightarrow \infty$. Using the \emph{a priori} bounds, we prove that, for every sequence of times $t_n \rightarrow \infty$, there is a convergent sub-sequence which must converge to a solution of the fixed point equation. We then prove that the solution of this fixed point is unique, has zero PDE residual, and is a solution to the PDE (\ref{PDE0}). Consequently, the trained neural network converges to the solution of the PDE. 

The rest of the paper is organized as follows. Our assumptions and main results are presented in Section \ref{S:MainResults}. Section \ref{S: PreliminaryCalculations} contains some preliminary calculations related to the clipped gradient descent algorithm which trains the neural network to minimize the PDE residual objective function. Section \ref{S:FixedPoints} establishes that fixed points of the neural network's limit PDE are weak solutions in the appropriate space to the PDE (\ref{PDE0}). In Section \ref{S:GlobalConvergenceNN_LimitPDE} we prove that the PDE residual converges to zero as first $N\rightarrow\infty$ and then $t\rightarrow\infty$. That is, the limit objective function $J(t)$ converges to zero as $t\rightarrow\infty$ and that the neural network converges to a global minimizer which is the solution of the PDE (\ref{PDE0}). Section \ref{S:ConvergenceLimitPDE_LargeN} proves the large $N$ limit of the trained neural network. We conclude with Section \ref{S:GlobalConvergenceFinitN_large_t} where we establish uniform-in-time error bounds on the objective function $J^{N}(\theta(t))$ for the PDE residual of the finite neural network, which is possible due to the monotonicity of $J^{N}(\theta(t))$. In particular, we are able to show that $\lim_{N \rightarrow \infty} \lim_{t \rightarrow \infty} J^{N}(\theta(t)) =  \lim_{t \rightarrow \infty} \lim_{N \rightarrow \infty} J^{N}(\theta(t)) = 0$.

\section{Assumptions and Main Results}\label{S:MainResults}

In this section, we state our assumptions and present the main results of the paper.

\subsection{Main Assumptions}
We first impose proper assumptions on the nonlinearity $f(v,w)$ and on the PDE (\ref{PDE0}).   
\begin{assumption} \label{PDEassumptions}We assume that $f(v,w)$ and its first two derivatives are uniformly bounded (and measurable functions).  $\mathcal{L}$ is a uniformly linear elliptic operator with the dissipativity condition that its eigenvalues satisfy $\inf_n \lambda_n > 0$ (and therefore the coercivity condition holds for $\mathcal{L}$). Furthermore, the coefficients of $\mathcal{L}$ are smooth. $\Omega \subset \mathbb{R}^n$ is a compact domain with a smooth boundary. The function $g$ is bounded on $\Omega$. 
\end{assumption}

The smoothness assumption on the coefficients of the operator $\mathcal{L}$ can be relaxed, but we do not do so here for presentation purposes.

\begin{assumption} \label{PDEassumptions2}
The PDE (\ref{PDE0}) has a unique weak solution in $H_0^1$.
\end{assumption}

Define the operator $D_Q h = \mathcal{L} h + f_v(Q, Q_x) h + \sum_{i=1}^n f_{w_i}(Q, Q_x) \frac{\partial h}{\partial x_i}$. In the following assumption, we will assume that the operator $D_Q: H_{(0)}^2 \rightarrow L^2$ is surjective, where $H_{(0)}^2 = H^2 \cap H_0^1$. 

\begin{assumption} \label{PDEassumptions3Surjectivity}
For any $R \in L^2$ and $Q \in H_0^1$, the PDE
\begin{eqnarray}
D_Q h &=& R, \phantom{....} x \in \Omega, \notag \\
h &=& 0, \phantom{....} x \in \partial \Omega,
\label{SurjectivePDE}
\end{eqnarray}
has a solution $h \in H_{(0)}^2$.
%has a solution $h \in H_0^1$.
\end{assumption}
Since $h \in H^2$, $h$ will satisfy $D_Q h = R$ almost everywhere. A sufficient condition for Assumption \ref{PDEassumptions3Surjectivity} to hold is Assumption \ref{PDEassumptions} and the additional assumption that $f_v(v,w) \geq 0$. No additional assumptions on $f_w(v,w)$ are necessary. Under these assumptions, due to Theorems 8.9 and 8.12 in \cite{GT}, there exists a (unique) solution $h \in H^2_{(0)}$ to (\ref{SurjectivePDE}).

As discussed in the introduction section we will approximate the solution $u(x)$ to the PDE (\ref{PDE0}) with a neural network
\begin{eqnarray*}
Q^N(x; \theta) = \eta(x) N^{- \beta} \sum_{i=1}^N c^i \sigma( w^i \cdot x + b^i),
\end{eqnarray*}
where $\eta(x)$ is a smooth function which vanishes on the boundary $\partial \Omega$ but has positive normal derivatives on the boundary and  $N^{-\beta}$ is a normalization factor with $\frac{1}{2} < \beta < 1$. 
For the trainable parameters $\theta = (c^i, w^i, b^i)_{i=1}^N$ and $\eta(x)$ we make the following assumption.
\begin{assumption} \label{NeuralNetworkAssumptions}
The neural network parameters $(c^i, w^i, b^i)$ are i.i.d. randomly initialized with distribution $\mu_0(dc, dw, db)$. The random variables $w^i, b^i$ have full support on $\mathbb{R}^{n+1}$ where $w^i$ has bounded moments, $c^i$ are mean-zero, bounded random variables and $c^i, w^i,$ and $b^i$ are independent. The neural network activation function $\sigma$ is bounded and has at least three bounded derivatives. $\sigma$ is a discriminatory function (such as a sigmoid or tanh function). $\eta \in C^3_b$ (with bounded derivatives) where $\eta > 0$ in the interior of $\Omega$ and $\eta = 0$ for $ x \in \partial \Omega$. In addition, $\nabla \eta \cdot \mathbf{n}_x \neq 0$ for $x \in \partial \Omega$ where $\mathbf{n}_x$ is the outward unit normal vector at $x$. 
\end{assumption}

Finally for the measure $\mu$, which is used in the objective function (\ref{PDEResidualObjectiveFunction0000}) we have Assumption \ref{Ass:Mu_measure}. 
\begin{assumption}\label{Ass:Mu_measure}
Assume that $\mu(dx)$ is a uniform probability measure on $\Omega$. 
\end{assumption}

We note that Assumption \ref{Ass:Mu_measure} is mainly posed for convenience and concreteness, and can be relaxed to any positive probability density function on $\Omega$ that is bounded away from zero in $\Omega$.  

\subsection{Summary of Main Results}

Let us now discuss the main results of the paper. Theorem \ref{ConvergenceNtoLimitPDE} proves that the neural network $Q^N(y; \theta(t))$ trained with clipped gradient descent (\ref{ActualClippedGradientDescentTrainingAlgorithm}) will converge as $N\rightarrow\infty$ to the solution of the nonlocal, nonlinear PDE
\begin{eqnarray*}
\frac{\partial Q(t, x )}{\partial t} &=& - \alpha  \int_{\Omega} U_{Q(t)}(x,y)  \big{(} \mathcal{A}_{Q(t) } Q(t, y ) - g(y) \big{)} \mu(dy),
\label{QlimitEvolutionSummaryofMainResults}
\end{eqnarray*}
 with the initial condition $Q(0,x) = 0$ and the kernel function 
\begin{eqnarray*}
U_{Q(t)}(x,y) &=& \bigg{\langle}  \nabla_{\theta}[ \eta(x) c \sigma (w x + b) ] \cdot  D_{Q(t)} [ \nabla_{\theta} ( \eta(y) c \sigma (w y + b) ) ] , \mu_0 \bigg{\rangle}.
\end{eqnarray*}
In addition, due to Theorem \ref{ConvergenceNtoLimitPDE}, the objective function $J^N(\theta(t))$ will converge as $N \rightarrow \infty$ to the objective function
\begin{eqnarray*}
J(t) = \frac{1}{2} \int_{\Omega} ( \mathcal{A}_{Q(t) } Q(t,x) - g(x) )^2 \mu(dx),
\end{eqnarray*}
which is the PDE residual of the limit neural network $Q(t,x)$. In the above equations, we have used the notation $\mathcal{A}_{Q(t) } Q(t, y )=[\mathcal{A}_{Q(t) } Q](t, y)$. Similar notation will be used for $D_{Q(t)} g(x,\theta) = [D_{Q(t)} g ](x,\theta)$.

The limit equation stated above for $Q(t,x)$ governs the evolution of the (limit) neural network output as a function of training time $t$. Theorem \ref{ConvergenceNtoLimitPDE} establishes that $\partial_{x}^i Q^N(t,x)$ converges almost surely to $\partial_{x}^i  Q(t,x)$ on $t \in [0,T]$ as $N \rightarrow \infty$ for $|i| \leq 2$:
\begin{eqnarray*}
\sup_{t \in [0,T] } \sup_{x \in \Omega} | \partial_{x}^i Q^N(t,x) - \partial_{x}^i Q(t,x) | \overset{a.s.} \rightarrow 0,
\end{eqnarray*}
where $T < \infty$ is arbitrary. As a direct consequence, for arbitrarily chosen $T<\infty$, as $N \rightarrow \infty$,
\begin{eqnarray*}
\sup_{t\in [0,T]}\norm{ Q^N(t) -  Q(t) }_{H^2} \overset{a.s.} \rightarrow 0.
\end{eqnarray*}

Theorem \ref{MainTheoremPDEResidualGlobalConvergence} proves that the limit neural network converges to a global minimizer. In particular, the limit objective function $J(t)$, which is the PDE residual of the neural network, converges to zero as the training time $t \rightarrow \infty$. 

Theorem \ref{MainTheoremPDESolutionGlobalConvergence} proves that the limit neural network $Q(t,x)$ converges to the solution $u(x)$ of the PDE (\ref{PDE0}) in $H^1$ as $t \rightarrow \infty$:
\begin{eqnarray}
\lim_{t \rightarrow \infty} \norm{ Q(t) - u }_{H^1} = 0.
\end{eqnarray}

Theorem \ref{FiniteNetworkPDEresidualUniformErrorBound} also establishes uniform-in-time error bounds for the objective function for the finite neural network with $N$ hidden units. For any $\epsilon > 0$, there exists a $t \geq 0$ and an $N_0<\infty$ such that
\begin{eqnarray}
 \mathbb{E}[ J^N(\theta(s)) ] \leq \epsilon,
\end{eqnarray}
for all $s \geq t$ and $N \geq N_0$. Consequently, if we make the neural network sufficiently large, we can achieve an arbitrarily low PDE residual. We highlight that such uniform-in-time estimates are typically difficult to achieve in the analysis of mean field systems. That is, typically mean-field analysis can only establish a result such as $\lim_{s \rightarrow \infty} \lim_{N \rightarrow \infty}  \mathbb{E}[ J^N(\theta(s)) ] = 0$, which is directly implied by our Theorems \ref{ConvergenceNtoLimitPDE} and \ref{MainTheoremPDESolutionGlobalConvergence}.  

 In Theorem \ref{FiniteNetworkPDEresidualUniformErrorBound}, we are able to additionally establish that the limits are interchangeable, i.e.,
\begin{eqnarray}
\lim_{N \rightarrow \infty} \lim_{t \rightarrow \infty}  \mathbb{E}[ J^N( \theta(t))]  = 0,
\end{eqnarray}
where the proof leverages the fact that the pre-limit objective function $J^N(\theta(t))$ is monotonically decreasing in $t$.

\section{Gradient Descent Algorithm for DGM and PINNs}\label{S: PreliminaryCalculations}

In this section we present preliminary calculations deriving the formulas for the evolution of the objective function $J^{N}(\theta(t))$ and the neural network $Q^N(x; \theta(t))$ as a function of the training time $t$. An important ingredient in our subsequent proofs will be an appropriate clipping of the gradient of the objective function used in the gradient descent algorithm. Gradient clipping is a standard method used with gradient descent algorithms in machine learning \cite{goodfellow}. Here, we shall see that it is important also for the convergence analysis to hold as $N \rightarrow \infty$. We emphasize though that while implementing gradient clipping is important for establishing the necessary \emph{a priori} bounds for the analysis, its effect vanishes as $N\rightarrow\infty$: the clipping does not appear in the limit training dynamics (\ref{QlimitEvolutionSummaryofMainResults}). 

The gradient of the objective function is
\begin{align}
\nabla_{\theta} J^N(\theta) &= \int_{\Omega} ( \mathcal{A}_{Q^N } Q^N(x; \theta) - g(x) ) \times \nabla_{\theta} \bigg{[} \mathcal{A}_{Q^N } Q^N(x; \theta) \bigg{]}  \mu(dx) \notag \\
&= \int_{\Omega} ( \mathcal{A}_{Q^N } Q^N(x; \theta) - g(x) ) \times D_{Q^N} \nabla_{\theta} Q^N(x; \theta) \mu(dx),
\label{ObjGrad12}
\end{align}
where the differential operator $D_{Q}$ is
\begin{align}
D_{Q} h &= \mathcal{L} h + f_v(Q, \partial_x Q) h + \sum_{i=1}^n f_{w_i}(Q,  Q_x) \frac{\partial h}{\partial x_i}.  
\end{align}

For example, denoting $Q^N(t) = Q^N(\cdot; \theta(t))$, we get 
\begin{align}
D_{Q^N(t)} h &= \mathcal{L} h + f_v(Q^N(x; \theta(t)),  Q^N_x(x; \theta(t))) h \notag \\
&+ \sum_{i=1}^n f_{w_i}(Q^N(x; \theta(t)), Q^N_x(x; \theta(t))) \frac{\partial h}{\partial x_i }.\notag 
\end{align}

The gradient with respect to a single parameter triplet $\theta_i = (c^i, w^i, b^i)$ is:
\begin{align}
\nabla_{\theta_i} J^N(\theta) &= N^{-\beta} \bigg{[} \int_{\Omega} ( \mathcal{A}_{Q^N } Q^N(x; \theta) - g(x) ) \times D_{Q^N} \nabla_{\theta_i} [ \eta(x) c^i \sigma(w^i \cdot x + b^i)  ] \mu(dx) \bigg{]}. 
\label{ObjGrad12i}
\end{align}
We will apply a clipping function to the term inside the brackets for the clipped gradient descent algorithm used to train the neural network. 

We introduce a clipping function $\phi^N$ which satisfies the following properties:
\begin{definition} [Smooth clipping function]
\label{psi}
A function class $\{\phi^N\}_{N\in \mathbb{N}^+}$ forms a family of smooth clipping functions with parameter $\gamma > 0$ if for any $N\in \mathbb{N}^+$  \begin{itemize}
    \item $\phi^N \in C^2_b(\mathbb{R})$ is increasing on $\mathbb{R}$.
    \item $|\phi^N|$ is bounded by $2N^\gamma$.
    \item $\phi^N(x)=x$ for $x \in [-N^\gamma,N^\gamma]$. 
    \item $|(\phi^N)'| \leq 1$ on $\mathbb{R}$.
\end{itemize}
\end{definition}
\begin{assumption}
\label{phi}
Functions $\{\phi^N\}_{N\in \mathbb{N}^+}$ are smooth clipping functions with parameters $\gamma$ where $\gamma>0$,  and for $\beta \in (\frac{1}{2},1)$ as in the definition of the neural network (\ref{Eq:NeuralNetwork}), we assume $ \gamma +\beta< 1$.
\end{assumption}

The training algorithm will clip the gradient (\ref{ObjGrad12i}), and the following clipped gradient  will be used to train the neural network:
\begin{align}
G^N_i(\theta) &=  N^{-\beta} \phi^N \bigg{(} \int_{\Omega} ( \mathcal{A}_{Q^N } Q^N(x; \theta) - g(x) ) \times D_{Q^N} \nabla_{\theta_i} [ \eta(x) c^i \sigma(w^i \cdot x + b^i) ] \mu(dx) \bigg{)},
\label{ObjGrad12iclipped}
\end{align}
where the scalar clipping function $\phi^N$ is applied element-wise to the vector in the above equation. (\ref{ObjGrad12iclipped}) can equivalently be written in the more general vector form
\begin{align}
G^N(\theta) &=  N^{-\beta} \phi^N \bigg{(} N^{\beta} \int_{\Omega} ( \mathcal{A}_{Q^N } Q^N(x; \theta) - g(x) ) \times D_{Q^N} \nabla_{\theta} Q^N(x; \theta) \mu(dx) \bigg{)}.
\label{ObjGrad12iclippedVector}
\end{align}

Notice that, due to Definition \ref{psi}, the effect of the clipping  vanishes as $N \rightarrow \infty$.

The parameters are trained with (clipped) gradient descent:
\begin{eqnarray}
\frac{d \theta}{d t} = - \alpha^N G^N(\theta(t)),
\label{ActualClippedGradientDescentTrainingAlgorithm}
\end{eqnarray}
where the learning rate $\alpha^N = \alpha N^{2 \beta - 1 }$ and $\alpha>0$ is a positive constant. 

The next step is to use the chain rule to derive an evolution equation for the neural network output,
\begin{align}
\frac{\partial Q^N(x; \theta(t) )}{\partial t} &= \nabla_{\theta} Q^N(x; \theta(t) ) \cdot \frac{d \theta}{dt} \notag \\
&=- \alpha^N \nabla_{\theta} Q^N(x; \theta(t) ) \cdot G^N(\theta(t)) \notag \\
&= - \alpha^N N^{- \beta} \phi^N \bigg{(} N^{\beta} \int  ( \mathcal{A}_{Q^N(t) } Q^N(y; \theta(t) ) - g(y) ) \notag \\
&\quad\times D_{Q^N(t)} [ \nabla_{\theta} Q^N(y; \theta(t) ) ]  \mu(dy) \bigg{)}  \cdot \nabla_{\theta} Q^N(x; \theta(t) ).
\label{PrelimitEvolutionA}
\end{align}

Next define the quantity  \begin{align}
B_t^N(x) &= \alpha^N N^{- \beta} \phi^N \bigg{(} N^{\beta} \int  ( \mathcal{A}_{Q^N(t ) } Q^N(y; \theta(t) ) - g(y) ) \notag \\
&\qquad\times D_{Q^N(t)} [ \nabla_{\theta} Q^N(y; \theta(t) ) ]  \mu(dy) \bigg{)}  \cdot \nabla_{\theta} Q^N(x; \theta(t) ), 
\end{align}
and the kernel 
\begin{align}
U^N_t(x,y) &=  N^{2 \beta - 1 } D_{Q^N(t)} \nabla_{\theta} Q^N(y; \theta(t) )  \cdot \nabla_{\theta} Q^N(x; \theta(t) ) \notag \\
&= \bigg{\langle}  \nabla_{\theta} [ \eta(x) c \sigma (w x + b) ] \cdot  D_{Q^N(t)}[ \nabla_{\theta}  ( \eta(y) c \sigma (w y + b) ) ], \mu_t^N \bigg{\rangle} , \notag 
\end{align}
where $\mu_t^N = N^{-1} \sum_{i=1}^N \delta_{\theta_i(t)}(d c, dw, db)$ is the empirical measure of the trained parameters. Then we obtain
\begin{align}
\frac{\partial Q^N(x; \theta(t) )}{\partial t} &= - \alpha \int_{\Omega} U_t^N(x,y) ( \mathcal{A}_{Q^N(\cdot; \theta(t) ) } Q^N(y; \theta(t) ) - g(y) ) \mu(dy) \notag \\
&\qquad+ \bigg{(} L_t^N(x) - B_t^N(x) \bigg{)},
\label{PrelimitEvolutionB}
\end{align}
where $L_t^N(x) = \alpha \int_{\Omega} U_t^N(x,y) ( \mathcal{A}_{Q^N(\cdot; \theta(t) ) } Q^N(y; \theta(t) ) - g(y) ) \mu(dy)$. We will later show that the term $L_t^N(x) - B_t^N(x)$ vanishes as $N \rightarrow \infty$ (i.e., the effect of the clipping function disappears as the neural network becomes larger). 

As $N \rightarrow \infty$, Theorem \ref{ConvergenceNtoLimitPDE} proves that $(U_t^N(x,y), Q^N(x; \theta(t)) )$ will converge to $(U_{Q(t)}(x,y), Q(t,x))$ (almost surely in $C^2(\Omega)$ for each $t \geq 0$) where:

\begin{eqnarray}
U_{Q}(x,y) = \bigg{\langle}  \nabla_{\theta} [ \eta(x) c \sigma (w x + b) ] \cdot  D_{Q} [ \nabla_{\theta} (\eta(y) c \sigma (w y + b)) ] , \mu_0 \bigg{\rangle},\notag
\end{eqnarray}
where $\mu_0$ is the distribution of the randomly initialized parameters at $t = 0$. The above kernel is similar to the overparameterized Neural Tangent Kernel (NTK) \cite{JacotGabrielHongler}, although there are two key differences: the kernel $U_Q(x,y)$ depends upon the PDE operator via $D_Q$ and it also depends nonlinearly upon the neural network $Q(t)$ (as well as its derivative $Q_x(t))$. The original NTK in \cite{JacotGabrielHongler} does not have PDE operators (since it is derived for a classic least squares regression problem), does not depend on the trained neural network solution, and does not depend upon time $t$ (i.e., it is a constant kernel that remains static during training). These are key distinctions which introduce non-trivial technical challenges to the analysis. The limit neural network's training evolution satisfies the PDE
\begin{align}
\frac{\partial Q(t, x )}{\partial t} &= - \alpha  \int_{\Omega} U_{Q(t)}(x,y)  \big{(} \mathcal{A}_{Q(t) } Q(t, y ) - g(y) \big{)} \mu(dy),
\label{QlimitEvolution}
\end{align}
 where the initial condition is $Q(0,x) = 0$. The PDE (\ref{QlimitEvolution}) governs the evolution of the (limit) neural network output as a function of training time $t$. (\ref{QlimitEvolution}) is a non-local nonlinear PDE due to the integral, the nonlinear dependence of $U_{Q(t)}$ on $(Q, Q_x)$, and the nonlinear dependence of $\mathcal{A}_{Q(t) }$ on $(Q, Q_x)$. We will sometimes use the notation $U_t(x,y) = U_{Q(t)}(x,y)$ for notational convenience.

Next, we use the notation established above to derive an equation for the evolution of the PDE residual of the neural network:
\begin{align}
&\frac{\partial}{\partial t}[ \mathcal{A}_{Q(t)} Q(t,x) - g](x) = D_{Q(t)} \bigg{[} \frac{\partial Q}{\partial t} \bigg{]}(t,x) \notag \\
&\qquad= - \alpha D_{Q(t)} \bigg{[}   \int_{\Omega} U_{Q(t)}(x,y)  \big{(} \mathcal{A}_{Q(t) } Q(t, y ) - g(y) \big{)} \mu(dy)  \bigg{]} . \notag \\
&\qquad = - \alpha \int_{\Omega} S_t(x,y) \big{(} \mathcal{A}_{Q(t ) } Q(t, y ) - g(y) \big{)} \mu(dy),
\label{ResidualEvolution}
\end{align}
where the kernel $S_{Q}(x,y)$ is defined to be
\begin{eqnarray}
S_{Q}(x,y) = \bigg{\langle} D_{Q} \bigg{[}   \nabla_{\theta} (\eta(x) c \sigma (w x + b))  \bigg{]} \cdot  D_{Q} \bigg{[}  \nabla_{\theta}( \eta(y) c \sigma (w y + b)) \bigg{]} , \mu_0 \bigg{\rangle}. \label{Eq:S_kernel}
\end{eqnarray}
We will sometimes use the notation $S_t(x,y) = S_{Q(t)}(x,y)$ for notational convenience.

Our calculations in Section \ref{S:ConvergenceLimitPDE_LargeN} will show that the pre-limit objective function $J^N(\theta(t) )$ will converge to the following limit objective function $J(t) = J( Q(t) )$ as $N \rightarrow \infty$:
\begin{eqnarray*}
J(t) = \frac{1}{2} \int_{\Omega} ( \mathcal{A}_{Q(t) } Q(t,x) - g(x) )^2 \mu(dx).
\end{eqnarray*}

Differentiating with respect to time yields
\begin{align}
\frac{d J}{d t} &= - \alpha \int_{\Omega^2}  ( \mathcal{A}_{Q(t) } Q(t,x) - g(x) ) S_{Q(t)}(x,y)  ( \mathcal{A}_{Q(t) } Q(t,y) - g(y) ) \mu(dx)  \mu(dy).
\label{ObjEvolution}
\end{align}

In Section \ref{S:FixedPoints} we prove that fixed points of the right hand side of (\ref{ObjEvolution}) are weak solutions to the PDE (\ref{PDE0}).

\section{Fixed Points of the Neural Network's Limit PDE}\label{S:FixedPoints}

In this section we study the fixed points $Q$ of (\ref{ObjEvolution}). We will return later to the question of the convergence of (\ref{QlimitEvolution}) as $t \rightarrow \infty$.

Define $R(t,x) = \mathcal{A}_{Q(t) } Q(t,x) - g(x) $ and $R(x) = \mathcal{A}_{Q } Q(x) - g(x) $. If $\frac{d J}{dt} = 0$, this implies that any fixed point $Q(x)$ and $R(x) = \mathcal{A}_{Q } Q(x) - g(x) $ must satisfy:
\begin{align}
0 &= \int_{\Omega^2}  R(x) S_Q(x,y)  R(y) \mu(dx)  \mu(dy),
\label{FixedPoint}
\end{align}
which can be lower bounded using the following calculations:
\begin{align}
0 &= \int_{\Omega^{2}}  R(x) S_Q(x,y)  R(y) \mu(dx)  \mu(dy) \notag \\
&= \int_{\Omega^{2}}  R(x) \langle D_{Q} [\eta(x) \sigma(w \cdot x + b )] D_{Q} [\eta(y) \sigma(w \cdot y + b )], \mu_0 \rangle R(y) \mu(dx)  \mu(dy)  \notag \\
&+ \int_{\Omega^{2}}  R(x)\langle D_{Q}[ \eta(x) c \sigma'(w \cdot x + b ) x] \cdot  D_{Q} [ \eta(y) c \sigma'(w \cdot y + b ) y ], \mu_0 \rangle R(y) \mu(dx)  \mu(dy)  \notag \\
&+ \int_{\Omega^{2}}  R(x)\langle D_{Q}[ \eta(x) c \sigma'(w \cdot x + b )]  D_{Q} [ \eta(y) c \sigma'(w \cdot y + b ) ], \mu_0 \rangle R(y) \mu(dx)  \mu(dy)  \notag \\
&= \mathbb{E}_{c, w, b \sim \mu_0} \bigg{[} \int_{\Omega^{2}}  R(x)  D_{Q}[ \eta(x)\sigma(w \cdot x + b )] D_{Q}[ \eta(y)\sigma(w \cdot y + b ) ] R(y) \mu(dx)  \mu(dy) \bigg{]}  \notag \\
&+ \mathbb{E}_{c, w, b \sim \mu_0} \bigg{[} \int_{\Omega^{2}}  R(x)  (D_{Q}[ \eta(x) c \sigma'(w \cdot x + b ) x] \nonumber\\
&\hspace{3cm}\cdot D_{Q} [\eta(y) c  \sigma'(w \cdot y + b ) y ]) R(y) \mu(dx)  \mu(dy) \bigg{]} \notag \\
&+ \mathbb{E}_{c, w, b \sim \mu_0} \bigg{[} \int_{\Omega^{2}}  R(x)  D_{Q}[ \eta(x) c\sigma'(w \cdot x + b )] D_{Q}[ \eta(y) c \sigma'(w \cdot y + b ) ] R(y) \mu(dx)  \mu(dy) \bigg{]}  \notag \\
&= \mathbb{E}_{c, w, b \sim \mu_0} \bigg{[} \bigg{(} \int_{\Omega}  R(x)  D_{Q}[ \eta(x) \sigma(w \cdot x + b ) ] \mu(dx)  \bigg{)}^2 \bigg{]} \notag \\
&+ \mathbb{E}_{c, w, b \sim \mu_0} \bigg{[} \norm{ \int_{\Omega}  R(x)  D_{Q} [ \eta(x) c \sigma'(w \cdot x + b ) x]  \mu(dx) }_2^2 \bigg{]} \label{ElementwiseSquare} \\
&+ \mathbb{E}_{c, w, b \sim \mu_0} \bigg{[} \bigg{(} \int_{\Omega}  R(x)  D_{Q}[ \eta(x) c \sigma'(w \cdot x + b ) ] \mu(dx)  \bigg{)}^2 \bigg{]} \notag \\
&\geq \mathbb{E}_{c, w, b \sim \mu_0} \bigg{[} \bigg{(} \int_{\Omega}  R(x)  D_{Q} [ \eta(x) \sigma(w \cdot x + b ) ]  \mu(dx)  \bigg{)}^2 \bigg{]}.
\label{LowerBoundRHS}
\end{align}

The latter calculation implies that 
\begin{align}
\mathbb{E}_{c, w, b \sim \mu_0} \bigg{[} \bigg{(} \int_{\Omega}  R(x)  D_{Q} [ \eta(x) \sigma(w \cdot x + b )  ] \mu(dx)  \bigg{)}^2 \bigg{]} &= 0.\label{LowerBoundRHS2}
\end{align}

In fact we have the following result
\begin{lemma}
Suppose that $R \in L^2$, $Q \in H_0^1$, and  that (\ref{LowerBoundRHS2}) holds.
Then, we have that for all $(w,b) \in \mathbb{R}^n \times \mathbb{R}$
\begin{eqnarray*}
\int_{\Omega}  R(x)  D_{Q} [ \eta(x) \sigma(w \cdot x + b ) ]  \mu(dx)  = 0.
\end{eqnarray*}
\end{lemma}
\begin{proof}
Equation (\ref{LowerBoundRHS2}) immediately implies that $E(w,b) = \int  R(x)  D_{Q} [ \eta(x) \sigma(w \cdot x + b )  ] \mu(dx)$ is zero almost everywhere in $(w,b) \in \mathbb{R}^n \times \mathbb{R}$. Since $E(w,b)$ is continuous (due to our assumptions and $R \in L^2$), $E(w,b) = 0$ everywhere. 
\end{proof}

Consequently, for any fixed point $R \in L^2$ and $Q \in H_0^1$,
\begin{eqnarray*}
\int_{\Omega}  R(x)  D_{Q} [ \eta(x) \sigma(w \cdot x + b ) ]  \mu(dx)  = 0,
\end{eqnarray*}
for all $(w,b) \in \mathbb{R}^n \times \mathbb{R}$.

Now, this implies that if $ R \in L^2$, then $\langle R, D_{Q} f \rangle = 0$ for any $f \in H_{(0)}^2 = H^2 \cap H_0^1$. 

\begin{lemma} \label{KeyLemma}
If $ R \in L^2$ and $Q \in H_0^1$ satisfy the fixed point equation (\ref{FixedPoint}), then $\langle R, D_{Q} f \rangle = 0$ for any $f \in H_{(0)}^2 = H^2 \cap H_0^1$.
\end{lemma}
\begin{proof}
Recall that
\begin{eqnarray}
\int_{\Omega}  R(x)  D_{Q} [ \eta(x) \sigma(w \cdot x + b )  ] \mu(dx)  = 0,
\label{ZeroEverywhere}
\end{eqnarray}
for all $(w,b) \in \mathbb{R}^n \times \mathbb{R}$. The linear span of $\{ \eta(x) \sigma( w \cdot x + b ) \}_{w,b \in \mathbb{R}^{n+1} }$ is dense in $H^2_{(0)}$; see \cite{QPDE} for the proof. Due to this density result, there exists a function $F^N(x) = \eta(x) \sum_{i=1}^N c^i \sigma( w^i \cdot x + b^i)$ which converges to any $f \in H^2_{(0)}$:
\begin{eqnarray*}
\lim_{N \rightarrow \infty} \norm{ F^N - f }_{H^2} = 0. 
\end{eqnarray*}

Then, due to equation (\ref{ZeroEverywhere}),
\begin{align}
\langle R, D_{Q} f \rangle &= \lim_{N \rightarrow \infty } \langle R, D_{Q} F^N \rangle \notag \\
&= \lim_{N \rightarrow \infty} \sum_{i=1}^N c^i \int_{\Omega} R(x) D_{Q} [ \eta(x) \sigma(w^i \cdot x + b^i ) ]\mu(dx) \notag \\
&= 0,\notag
\end{align}
concluding the proof of the lemma.
\end{proof}

\begin{remark} \label{RangeofDQ}
If $R \in L^2$ and for any $Q \in H_0^1$, then there exists an $h \in H^2_{(0)}$ such that $D_Q h = R$ (almost everywhere).
\end{remark}
\begin{proof}
This statement follows directly from Assumption \ref{PDEassumptions3Surjectivity}. In addition, a sufficient condition is $f_v(Q, \partial_x Q) \geq 0$. Recall that $D_Q h = \mathcal{L} h + f_v(Q, Q_x) h + \sum_{i=1}^n f_{w_i}(Q, Q_x) \frac{\partial h}{\partial x_i}$. Due to Assumption \ref{PDEassumptions}, $f_v$ and $f_{w_i}$ are bounded. Consider the PDE
\begin{eqnarray*}
D_{Q} h = R,
\end{eqnarray*}
where $h$ satisfies a Dirichlet boundary condition $h = 0$ on $\partial \Omega$ and the source term $R \in L^2$. Due to Theorems 8.9 and 8.12 in \cite{GT}, there exists a (unique) solution $h \in H^2_{(0)}$.  
\end{proof}

\begin{theorem} \label{GlobalMinimizerFixedPoint}
Any $(R,Q) \in L^2 \times H_0^1$ satisfying the fixed point characterized by equation (\ref{FixedPoint}) is a global minimizer with zero PDE residual:
\begin{eqnarray*}
\norm{ R}_{L^2} = 0.
\end{eqnarray*}
\end{theorem}
\begin{proof}
Due to Remark \ref{RangeofDQ}, there exists an $h \in H_{(0)}^2$ such that $D_{Q} h = R$ (almost everywhere). Therefore, for every test function $v \in H_0^1$,

\begin{eqnarray*}
 \int_{\Omega} \bigg{[} D_Q h(x) \bigg{]} v(x) dx = \int_{\Omega} R(x) v(x) dx. 
\end{eqnarray*}

Recall that $H_0^1$ is dense in $L^2(\Omega)$ (since $C_c^{\infty}$ is a subset of $H_0^1$). Consider a sequence of functions $v_m \in H_0^1$ where $\norm{ v_m - R}_{L^2} \rightarrow 0$. Then, for each $m$,
\begin{eqnarray*}
 \int_{\Omega} \bigg{[} D_Q h(x)  - R(x) \bigg{]} v_m dx = 0,
\end{eqnarray*}
and by the Cauchy-Schwarz inequality 
\begin{align*}
|  \int_{\Omega} ( D_Q h(x)  - R(x) ) v_m dx - \int_{\Omega} ( D_Q h(x)  - R(x) ) R(x) dx  | &\rightarrow 0,
\end{align*}
as $m \rightarrow \infty$. Consequently, we obtain $\int_{\Omega} ( D_Q h(x)  - R(x) ) R(x) dx  = 0$. 

We next apply Lemma \ref{KeyLemma} to get
\begin{align}
\norm{R}_{L^2}^2 &= \int_{\Omega} D_Q h(x) R(x) dx = 0.
\end{align}
Note that $h \in H_{(0)}^2$ and $R \in L^2$, which satisfy the assumptions of Lemma \ref{KeyLemma}. This concludes the proof of the theorem.

\end{proof}

Furthermore, under suitable technical conditions, if the PDE residual is zero, then the neural network function (at the fixed point) is a solution to the PDE (\ref{PDE0}).
\begin{corollary} \label{UniqueWeakSolution}
Under Assumption \ref{PDEassumptions} and Assumption \ref{PDEassumptions2}, if $Q \in H_0^1$ and $\norm{R}_{L^2}^2 = \norm{ \mathcal{A}_{Q} Q - g }_{L^2}^2 = 0$, then $Q$ is a unique weak solution in $H_0^1$ to the PDE (\ref{PDE0}).  
\end{corollary}
\begin{proof}

For any test function $\phi \in H_0^1$, we can apply the Cauchy-Schwarz inequality to show that
\begin{align}
 \langle \mathcal{A}_Q Q - g, \phi \rangle^2 &\leq\norm{ \mathcal{A}_Q Q - g }^2_{L^2} \times \norm{ \phi }^2_{L^2} \notag \\
&= 0.\notag
\end{align}

This implies that
\begin{eqnarray}
 \langle \mathcal{A}_Q Q - g, \phi \rangle = 0.\notag
\end{eqnarray}

Since $Q$ and $\phi$ are in $H_0^1$, the above equation can be integrated by parts to produce
\begin{align}
&  \int_{\Omega} \bigg{(} \sum_{i,j=1}^n a_{ij}(x) \frac{\partial \phi}{\partial x_i}(x)  \frac{\partial Q }{\partial x_j}(x)  + \sum_{i=1}^n c_i(x) \frac{\partial Q}{\partial x_i}( x)  \phi(x)  - d(x) Q( x) \phi(x)  \notag \\
&+ f( Q( x), Q_x(x) ) \phi(x) - g(x) \phi(x) \bigg{)} dx = 0. 
\label{WeakFormOfPDEFixedPoint}
\end{align}

The definition of a weak solution in $H_0^1$ is a $Q \in H_0^1$ which satisfies (\ref{WeakFormOfPDEFixedPoint}) for every $\phi \in H_0^1$. We recall that we have assumed in Assumption \ref{PDEassumptions2} that there is a unique weak solution to (\ref{WeakFormOfPDEFixedPoint}). This concludes the proof of the Corollary.
\end{proof}

\section{Global Convergence of Neural Network's Limit PDE as $t \rightarrow \infty$}\label{S:GlobalConvergenceNN_LimitPDE}

Now that we have characterized the fixed points of the neural network's limit PDE (\ref{QlimitEvolution}), we will analyze its convergence as $t \rightarrow \infty$. We first establish the monotonicity of the limit objective function $J(t)$. In our following calculations, $C$ will be used as a generic finite positive constant which may change from line to line. 

\begin{lemma} \label{RuniformBound}
$J(t)$ is monotone decreasing in $t$ and for the residual $R(t,x) = \mathcal{A}_{Q(t) } Q(t,x) - g(x) $ we have that there is a uniform in time constant $C<\infty$ such that $\sup_{t\geq 0}\norm{R(t,\cdot) }_{L^2} \leq C < \infty$. 
\end{lemma}
\begin{proof}
Using similar calculations as in equation (\ref{LowerBoundRHS}), we can establish the lower bound
\begin{align}
 & \int_{\Omega\times\Omega}  R(t,x) S_{Q(t)}(x,y)  R(t,y) \mu(dx)  \mu(dy) \geq 0.
\end{align}

Since $\frac{d J}{d t}= - \alpha \int_{\Omega\times\Omega}  R(t,x) S_{Q(t)}(x,y)  R(t,y) \mu(dx)  \mu(dy)$, we obtain
\begin{eqnarray}
\frac{d J}{d t} \leq 0. 
\end{eqnarray}
Consequently, we get $J(t) = \frac{1}{2} \norm{R(t,x) }^2_{L^2} \leq J(0) < \infty $.
\end{proof}

The above lemma establishes that the PDE residual of the neural network $Q(t,x)$ is uniformly bounded in $L^2$ in time:
\begin{eqnarray}
\sup_{t\geq 0}\norm{\mathcal{L} Q(t) + f(Q(t), Q_x(t) )}_{L^2}^2 \leq C < \infty,
\end{eqnarray}
where $C$ is uniform in time. Since $f(v,w)$ is uniformly bounded, we can use the triangle inequality to show that
\begin{align}
\norm{\mathcal{L} Q(t) }_{L^2} &=\norm{\mathcal{L} Q(t) + f(Q(t), Q_x(t)) -f(Q(t), Q_x(t) ) }_{L^2}  \notag \\
&\leq  \norm{\mathcal{L} Q(t) + f(Q(t), Q_x(t) ) }_{L^2} + \norm{f(Q(t), Q_x(t) )  }_{L^2}   \notag \\
&\leq C_1 < \infty,
\label{LQbound}
\end{align}
also uniformly in time. 

The following lemma shows that, if a solution to (\ref{QlimitEvolution}) exists in $C([0, T]; H_{(0)}^2)$, then it must satisfy an explicit quantitative bound on $[0,T]$. We will later prove existence and uniqueness of the solution in Lemma \ref{Hestimate0}. 

\begin{lemma} \label{Hestimate0apriori}
Suppose that there is a solution $Q(t,x)$ to the equation (\ref{QlimitEvolution}) in $C([0, T]; H_{(0)}^2)$. Then, there exists a constant $C(T)$ such that
\begin{eqnarray}
\sup_{t \in [0,T]} \norm{Q(t)}_{H^2} &\leq& C(T) \notag \\
&=& C_1 T \exp(C_2 T),
\end{eqnarray}
where $C_1$ and $C_2$ are positive finite constants.
\end{lemma}
\begin{proof}
Integrating (\ref{QlimitEvolution}) with respect to time and differentiating with respect to the spatial variables yields
\begin{align}
Q(t,x) &= - \alpha  \int_0^t \int_{\Omega} U_{Q(s)}(x,y)  \big{(} \mathcal{A}_{Q(s ) } Q(s, y ) - g(y) \big{)} \mu(dy), \notag \\
 \nabla_{x} Q(t, x ) &= - \alpha \int_0^t  \int_{\Omega} \nabla_x U_{Q(s)}(x,y)  \big{(} \mathcal{A}_{Q(s) } Q(s, y ) - g(y) \big{)} \mu(dy), \notag \\
 \nabla_{xx} Q(t, x ) &= - \alpha \int_0^t  \int_{\Omega} \nabla_{xx} U_{Q(s)}(x,y)  \big{(} \mathcal{A}_{Q(s) } Q(s, y ) - g(y) \big{)} \mu(dy).
\label{QlimitEvolutionDiff}
\end{align}

Define $V(t,x) =  | Q(t, x ) | + \sum_{i=1}^n | Q_{x_i}(t,x) | + \sum_{i,j=1}^n | Q_{x_i x_j} (t, x ) |$. Taking an absolute value, summing over $i,j = 1,\ldots, n$, and using the boundedness assumptions on $f$ produces
\begin{align}
& V(t,x) \leq   C_1  \int_0^t  \int_{\Omega} \max_{i,j} ( | U_{Q(s)}(x,y)| +  |\partial_{x_i} U_{Q(s)}(x,y) |  + | \partial_{x_i x_j} U_{Q(s)}(x,y)   | ) V(s,y)  dy ds \notag \\
&+ C_2  \int_0^t  \int_{\Omega} \max_{i,j} ( | U_{Q(s)}(x,y)| +  |\partial_{x_i} U_{Q(s)}(x,y) |  + | \partial_{x_i x_j} U_{Q(s)}(x,y)   | ) (1+|g(y)| )  dy ds. 
\label{Vbound11}
\end{align}

Here $\max_{i,j} ( | U_{Q(s)}(x,y)| +  |\partial_{x_i} U_{Q(s)}(x,y) |  + | \partial_{x_i x_j} U_{Q(s)}(x,y)   | )$ is uniformly bounded (in both the spatial variables $x,y \in \Omega$ and time $t \geq 0$) due to the assumptions on the moments of $w_0^i$, the derivatives of $\sigma$ being bounded, and $\Omega$ being a compact domain. Therefore, the last term in (\ref{Vbound11}) is bounded via the Cauchy-Schwarz inequality since $g \in L^2$. Due to the assumption that $Q(t) \in H^2$,  we can also apply the Cauchy-Schwarz inequality to (\ref{QlimitEvolutionDiff}) to show that $(Q(t,x), \nabla_{x} Q(t, x ), \nabla_{xx} Q(t, x ))$ exist and are finite for each $(t,x) \in (0,T) \times \Omega$. 

This allows us to take a supremum first on the RHS and then on the LHS of (\ref{Vbound11}): 
\begin{align}
\sup_{x \in \Omega} V(t,x) &\leq   C_1 \int_0^t  \int_{\Omega} \sup_{y \in \Omega} V(s,y)  dy ds + C_2 t \notag \\
&\leq   C_3 \int_0^t \sup_{y \in \Omega} V(s,y)  ds + C_2 t.\nonumber
\end{align}

By Gronwall's inequality, 
\begin{eqnarray}
\sup_{x \in \Omega} V(t,x) &\leq C_0 t \exp( C_2 t ),
\label{UniformVBound0}
\end{eqnarray}
where $C_0$ and $C_2$ are positive finite constants. This of course implies that $\norm{ Q(t,x)}_{H^2} \leq C(T) = C_1 T \exp( C_2 T )$ for every $t \in [0,T]$. 
\end{proof}

We will now prove existence and uniqueness of a solution to (\ref{QlimitEvolution}) in $C([0, T]; H_{(0)}^2)$. 

\begin{lemma} \label{Hestimate0}
There exists a unique solution $Q(t,x)$ to the equation (\ref{QlimitEvolution}) in $C([0, T]; H_{(0)}^2)$, where $T > 0$ is arbitrary. 
\end{lemma}
\begin{proof}
We first consider the RHS of equation (\ref{QlimitEvolution}):
\begin{align}
G_{t, \Delta}(q)(\tau, x) &= Q(t,x) - \alpha \int_0^{\tau}  \int_{\Omega} U_{q(s)}(x,y)  \big{(} \mathcal{A}_{q(s) } q(s,y) - g(y) \big{)} \mu(dy) ds ,
\label{QlimitEvolutionRHS5}
\end{align}
where $0 \leq \tau \leq \Delta$. Then, we have that
\begin{align}
& |G_{t, \Delta}(Q_1)(\tau,x) - G_{t, \Delta}(Q_2)(\tau,x) | \leq C  \int_0^{ \tau} \int_{\Omega} \big{|} \mathcal{A}_{Q_1 } Q_1 - \mathcal{A}_{Q_2 } Q_2 \big{|} \mu(dy) ds \notag \\
&\quad+ C  \int_0^{ \tau} \int_{\Omega} (| Q_1 - Q_2 | + \sum_{i=1}^n | \frac{\partial Q_1}{\partial x_i} - \frac{\partial Q_2}{\partial x_i} | )\times  \big{|} \mathcal{A}_{Q_2 } Q_2 - g \big{|} \mu(dy) ds \notag \\
&\leq \int_0^{ \tau} C  ( 1 + \norm{ \mathcal{A}_{Q_2 } Q_2 - g  }_{L^2} ) \times \norm{ Q_1 - Q_2 }_{H^2} ds. \notag \\
&\leq \Delta C  ( 1 + \sup_{s \in [0,  \Delta] } \norm{ \mathcal{A}_{Q_2 } Q_2(s) - g  }_{L^2} ) \notag \\
&\quad\times \sup_{s \in [0, \Delta] } \norm{ Q_1(s) - Q_2(s) }_{H^2},
\label{QlimitEvolutionRHS6}
\end{align}
where we have used the bounded moments of $w_0^i$ and the uniform boundedness of $f_{vv}, f_{vw},$ and $f_{ww}$ in the second line. Similarly, bounds can be established on the first and second derivatives $| \partial_{x_i} G_{t, \Delta}(Q_1)(\tau,x) - \partial_{x_i} G_{t, \Delta}(Q_2)(\tau,x) |$ and $| \partial_{x_i x_j} G_{t, \Delta}(Q_1)(\tau,x) - \partial_{x_i x_j} G_{t, \Delta}(Q_2)(\tau,x) |$, yielding
\begin{align}
&\sup_{s \in [0, \Delta ] } \norm{ G_{t, \Delta}(Q_1)(s) - G_{t, \Delta}(Q_2)(s) }_{H^2} \leq \nonumber\\
&\leq \Delta \times C_1  ( 1 + \sup_{s \in [0,  \Delta] } \norm{ \mathcal{A}_{Q_2 } Q_2(s) - g  }_{L^2} ) \times \sup_{s \in [0, \Delta] } \norm{ Q_1(s) - Q_2(s) }_{H^2}. 
\label{QlimitEvolutionRHS7}
\end{align}

Using similar calculations,
\begin{align}
\sup_{s \in [0, \Delta ] } \norm{ G_{t, \Delta}(q)(s) }_{H^2} &\leq \norm{Q(t)}_{H^2} +  \Delta \times C_2  \times \sup_{s \in [0, \Delta] } \norm{ q(s) }_{H^2}. 
\label{QlimitEvolutionRHS8}
\end{align}

Note that $C_1, C_2$ do not depend upon time $t$. Let $C_3 = \max(C_1, C_2)$. Define the map $M: C([0, \Delta]; H^2) \rightarrow C([0, \Delta]; H^2)$ where $M(q)(s,x) = G_{0, \Delta}(q)(s,x)$. Define the set $B = \{Q: \norm{ Q }_{C([0, \Delta]; H^2)} \leq 2C(T) \}$. We will consider the iteration $\bar{Q}^{k+1} = M( \bar{Q}^k )$ with $\bar{Q}^0 \in B$. 

Consider a time interval $[0,T]$, where from our previous estimates (in Lemma \ref{Hestimate0apriori}) any solution must satisfy $\norm{Q(t)}_{H^2} \leq C(T)$ for $t \in [0,T]$. Then, due to the bound (\ref{QlimitEvolutionRHS8}), $M(q) \in B$ for any $q \in B$ if $\Delta \leq \frac{1}{C_3 \times 4C(T)}$. 

Consider two functions $Q_1, Q_2 \in C([0, \Delta]; H^2)$ where $\sup_{t \in [0, \Delta] } \norm{Q_1(t) }_{H^2} \leq 2 C(T)$ and $\sup_{t \in [0, \Delta] } \norm{Q_2(t) }_{H^2} \leq 2 C(T)$. From (\ref{QlimitEvolutionRHS7}), if we select $\Delta$ sufficiently small, we can establish a contraction $\norm{M(Q^1)- M(Q^2)}_{C([0, \Delta]; H^2)} \leq \ell \norm{Q^1 - Q^2}_{C([0, \Delta]; H^2)}$ for $0 < \ell < 1$ and $Q^1, Q^2 \in B$. 

By the Banach fixed point theorem, $M(q)$ has a unique fixed point in $B$. Therefore, there is a unique solution on the time interval $[0, \Delta]$. Once a unique solution is established on $[0, \Delta]$, we can -- using the same method -- prove a unique solution on the time intervals $[\Delta, 2 \Delta], [2 \Delta, 3 \Delta], \ldots$. For the subsequent time intervals, e.g. $[i \Delta, (i+1) \Delta]$, the map is chosen to be $M(q)(s,x) = G_{i \Delta, (i+1) \Delta}(q)(s,x)$ for $s \in [0, \Delta]$ with the initial condition in (\ref{QlimitEvolutionRHS5}) being the unique solution $Q(i \Delta,x)$ (from the previous time interval). The contraction constant $\ell$ and the set $B$ will remain the same (for all time intervals $[i \Delta, (i+1) \Delta]$ of $[0,T]$), and the fixed point of the map $M$ will be the unique solution on the time interval $[i \Delta, (i+1) \Delta]$. 

Finally, we note that $Q(t,x)$ is in $H^2_{(0)}$ for all $t \geq 0$ since $Q(0,x) = 0$ and $U_{Q(t)}(x,y) = 0$ for $x \in \partial \Omega$, which leads to $Q(t,x) = 0$ for $x \in \partial \Omega$. This completes the proof. 

\end{proof}

Lemma \ref{Hestimate0} proves that there is a unique solution to equation (\ref{QlimitEvolution}) in $C([0, T]; H_{(0)}^2)$. Furthermore, it can be easily verified that the solution's time derivative exists and is finite at every $(t,x) \in (0,T) \times \Omega$. Since $Q(t,x) \in C([0, T]; H_{(0)}^2)$, $A(s,x) =\int_{\Omega} U_{Q(s)}(x,y)  \big{(} \mathcal{A}_{Q(s) } Q(s, y ) - g(y) \big{)} \mu(dy)$ is continuous in time $s$ for each $x \in \Omega$. Since $Q(t,x) = Q(0,x) - \alpha \int_0^t A(s,x) ds$, $Q(t,x)$ has a time derivative which satisfies $\frac{dQ}{dt} = - \alpha A(t,x)$ by the fundamental theorem of calculus.

We are now in position to prove a uniform $H^2$ bound on $Q(t,x)$:
\begin{lemma} \label{Qbound}
$Q(t,x)$ is uniformly bounded in $H^2$:
\begin{eqnarray}
\norm{ Q(t) }_{H^2} \leq C_2 < \infty. 
\end{eqnarray}
\end{lemma}
\begin{proof}
Due to Lemma \ref{Hestimate0}, $Q(t) \in H_{(0)}^2$. Similar to equation (\ref{LQbound}), we can show that $Q(t)$ is a solution to a PDE $\mathcal{L} Q(t) = q(t)$ where $q(t) \in L^2$. Specifically, $\mathcal{L} Q(t) + f( Q(t), Q_x(t) ) = R(t)$. Re-arranging, $\mathcal{L} Q(t) =  -f( Q(t), Q_x(t) ) + R(t)$. Due to $f(v,w)$ being uniformly bounded and $\norm{R(t)}_{L^2}$ also being uniformly bounded, $Q(t)$ satisfies the PDE $\mathcal{L} Q(t) = q(t)$ where $q(t)$ is uniformly-in-time bounded (see Lemma \ref{RuniformBound}). Therefore, from Theorem 8.12 (page 186) in  \cite{GT},
\begin{align}
\norm{ Q(t) }_{H^2}  &\leq K ( \norm{\mathcal{L} Q(t) }_{L^2} + \norm{ Q(t) }_{L^2} ) \notag \\
&\leq K_1 + K \norm{ Q(t) }_{L^2}.\notag
\end{align}

It only remains to establish the $L^2$ bound on $Q(t)$. Due to our assumptions $\mathcal{L}$ satisfies the coercivity condition:
\begin{eqnarray}
K_2 \norm{ Q(t) }^2_{L^2} \leq \langle \mathcal{L} Q(t), Q(t) \rangle,\notag
\end{eqnarray}
where $K_2 > 0$. Applying the Cauchy-Schwarz inequality on the RHS,
\begin{eqnarray}
K_2 \norm{ Q(t) }^2_{L^2} \leq \norm{ \mathcal{L} Q(t)}_{L^2}  \norm{ Q(t) }_{L^2}.\notag
\end{eqnarray}

Dividing by $\norm{ Q(t) }_{L^2}$ for the case where $\norm{ Q(t) }_{L^2} > 0$, we have the uniform-in-time bound
\begin{align}
 \norm{ Q(t) }_{L^2} &\leq\frac{1}{K_2} \norm{ \mathcal{L} Q(t)}_{L^2} \leq \frac{C_1}{K_2},\notag
\end{align}
where the second inequality is due to (\ref{LQbound}). When $\norm{ Q(t) }_{L^2} = 0$, we also of course have that $ \norm{ Q(t) }_{L^2}  \leq C_2$ where $C_2 = \frac{C_1}{K_2} > 0$.\end{proof}

\begin{lemma} \label{StronglyConvergentSubseq}
For any sequence of times $t_n \rightarrow \infty$, there is a convergent subsequence $Q(t_{n_k}) \in H^1$. 
\end{lemma}
\begin{proof}
By the Rellich-Kondrachov theorem (see Theorem 7.29 in \cite{RenardyRogers}), the bounded set $A = \{ u \in H^2: \norm{ u}_{H^2} \leq C \}$ is precompact in $H^1$. Since $\norm{ Q(t_n) }_{H^2} \leq C$, $Q(t_n) \in A$ where $A$ is precompact in $H^1$. Therefore, there is a further subsequence $Q(t_{n_k})$ which converges to a limit $Q$ in $H^1$ (where the limit point may depend upon the sequence $t_{n_k}$).
\end{proof}

\begin{lemma} \label{WeaklyConvergentSubseq}
For any sequence of times $t_n \rightarrow \infty$, there is a weakly convergent subsequence $R(t_{n_k},x)$ in $L^2$.
\end{lemma}
\begin{proof}
Since $R(t_n)$ is uniformly bounded in $L^2$, there is a weakly convergent subsequence $R(t_{n_k}) \overset{w} \rightarrow R \in L^2$ (Theorem 3, Appendix D, page 639 in \cite{Evans}).
\end{proof}

We now establish a uniform bound on $\frac{\partial Q}{\partial t}$ (and its first and second spatial derivatives) which will be essential for the subsequent analysis. 
\begin{lemma} \label{QTimederivativeUniformBound}
$\sup_{x \in \Omega} |\frac{\partial Q}{\partial t}(t,x) | $, $\sup_{x \in \Omega} |\frac{\partial Q_{x_i}}{\partial t}(t,x) | $, and $\sup_{x \in \Omega} |\frac{\partial Q_{x_i x_j}}{\partial t}(t,x) | $ are uniformly bounded in time $t \geq 0$. 
\end{lemma}
\begin{proof}
Recall that $Q(t,x)$ satisfies the evolution equation
\begin{align}
\frac{\partial Q(t, x )}{\partial t} &= - \alpha  \int_{\Omega} U_{Q(t)}(x,y)  R(t,y) \mu(dy), \notag \\
U_t(x,y) &= \bigg{\langle}  \nabla_{\theta} [ \eta(x) c \sigma (w x + b) ] \cdot D_{Q(t)} [\nabla_{\theta} ( \eta(y) c \sigma (w y + b) ) ] , \mu_0 \bigg{\rangle}, \notag
\end{align}
where the initial condition is $Q(0,x) = 0$. Then,
\begin{align}
|\frac{\partial Q(t, x )}{\partial t}|^2 &\leq \alpha^2 \bigg{(}  \int_{\Omega} |U_{Q(t)}(x,y)|  |R(t,y)| \mu(dy) \bigg{)}^2, \notag \\
&\leq C \int_{\Omega} |U_{Q(t)}(x,y)|^2 \mu(dy) \times \int_{\Omega} |R(t,y)|^2 \mu(dy) \notag \\
&\leq C < \infty,
\end{align}
where we have applied Assumption \ref{PDEassumptions} and Assumption \ref{NeuralNetworkAssumptions} as well as the bound from Lemma \ref{RuniformBound}.

Similarly, we can calculate uniform bounds for the first and second spatial derivatives of $\frac{\partial Q}{\partial t}$. Details are omitted due to the similarity of the argument.
\end{proof}

\begin{lemma} \label{BarbalatLemma}
We have that
\begin{eqnarray}
\lim_{t \rightarrow \infty} \langle R(t), S_{Q(t)} R(t) \rangle = 0. \nonumber
\end{eqnarray}
\end{lemma}
\begin{proof}
Recall that
\begin{align}
\frac{d J}{dt} &= - \alpha \langle R(t), S_{Q(t)} R(t) \rangle.
\end{align}

We will apply Barbalat's Lemma (see page 205, Section 5.5 of \cite{NonlinearSystems}) to prove that $\lim_{t \rightarrow \infty} \langle R(t), S_{Q(t)} R(t) \rangle = 0$, which requires showing that (a) $\lim_{t \rightarrow \infty} J(t)$ exists and is finite and (b) $\frac{d J}{dt} = - \alpha \langle R(t), S_{Q(t)} R(t) \rangle$ is uniformly continuous in time $t$. 

First, we show that $\frac{d J}{dt}(t)$ is uniformly continuous. Due to Lemma \ref{Hestimate0}, $Q(t,x) \in C([0, \infty); H_{(0)}^2)$. We start by considering the following decomposition 
\begin{align}
& \langle R(t + \delta), S_{t+ \delta} R(t + \delta) \rangle - \langle R(t ), S_{t} R(t ) \rangle = \langle R(t + \delta), (S_{t+ \delta} - S_t) R(t + \delta) \rangle \notag \\
&\quad+ \langle R(t + \delta), S_t R(t + \delta) \rangle  - \langle R(t ), S_{t} R(t ) \rangle \notag \\
&=  \langle R(t + \delta), (S_{t+ \delta} - S_t) R(t + \delta) \rangle + \langle R(t + \delta), S_t (R(t + \delta) - R(t)) \rangle \notag \\
&\quad + \langle R(t + \delta), S_t R(t) \rangle - \langle R(t ), S_{t} R(t ) \rangle \notag \\
&= \langle R(t + \delta), (S_{t+ \delta} - S_t) R(t + \delta) \rangle + \langle R(t + \delta), S_t (R(t + \delta) - R(t) ) \rangle \notag \\
&\quad + \langle R(t + \delta) - R(t), S_t R(t) \rangle,
\end{align}
where for notational convenience we have denoted $S_t = S_{Q(t)}$. We next bound each of the terms in the last line above. Due to Assumption \ref{PDEassumptions} and Assumption \ref{NeuralNetworkAssumptions}, we have the uniform bound $| S_t(x,y) | \leq C_1$. Recall that $\norm{R(t)}_{L^2} \leq C_0$ and that we have also proven that $\norm{Q(t,x) }_{H^2} < C_2$ in Lemma \ref{Qbound}. Using the Cauchy-Schwarz inequality,
\begin{align}
| \langle R(t + \delta) - R(t), S_t R(t) \rangle |^2 &\leq \norm{ R(t + \delta) - R(t)}_{L^2}^2  \norm{ S_t R(t) }_{L^2}^2 \notag \\
&\leq C_3 \norm{ R(t + \delta) - R(t)}_{L^2}^2  \norm{ R(t) }_{L^2}^2 \notag \\
&\leq C_4 \norm{ Q(t + \delta) - Q(t) }_{H^2}^2.
\end{align}

Using similar calculations
\begin{align}
|  \langle R(t + \delta), S_t (R(t + \delta) - R(t) ) \rangle |^2 &\leq C_5 \norm{ Q(t + \delta) - Q(t) }_{H^2}^2.
\end{align}

Using the Cauchy-Schwarz inequality,
\begin{align}
\langle R(t + \delta), (S_{t+ \delta} - S_t) R(t + \delta) \rangle^2 &\leq \norm{R(t + \delta)}_{L^2}^2 \norm{ (S_{t+ \delta} - S_t) R(t + \delta)}_{L^2}^2. \notag
\end{align}

We first apply the following re-arrangement of terms:
\begin{align}
&S_{t+\delta}(x,y) - S_t(x,y) =\nonumber\\
&=\bigg{\langle} D_{Q(t + \delta)} \bigg{[}   \nabla_{\theta} ( \eta(x) c \sigma (w x + b) ) \bigg{]} \cdot  D_{Q(t + \delta)} \bigg{[} \nabla_{\theta} ( \eta(y) c \sigma (w y + b)  ) \bigg{]} , \mu_0 \bigg{\rangle} \notag \\
&\quad-\bigg{\langle} D_{Q(t)} \bigg{[}   \nabla_{\theta} ( \eta(x) c \sigma (w x + b) ) \bigg{]} \cdot   D_{Q(t)}  \bigg{[}   \nabla_{\theta} ( \eta(y) c \sigma (w y + b) ) \bigg{]} , \mu_0 \bigg{\rangle} \notag \\
&= \bigg{\langle} (D_{Q(t + \delta)} - D_{Q(t)} ) \bigg{[}   \nabla_{\theta} ( \eta(x) c \sigma (w x + b) ) \bigg{]} \cdot  D_{Q(t + \delta)} \bigg{[} \nabla_{\theta} ( \eta(y) c \sigma (w y + b) ) \bigg{]} , \mu_0 \bigg{\rangle} \notag \\
&\quad+\bigg{\langle} D_{Q(t)} \bigg{[}   \nabla_{\theta} ( \eta(x) c \sigma (w x + b) ) \bigg{]} \cdot  ( D_{Q(t + \delta)} -  D_{Q(t)} ) \bigg{[}   \nabla_{\theta} ( \eta(y) c \sigma (w y + b) ) \bigg{]} , \mu_0 \bigg{\rangle}.\nonumber
\end{align}

Using a Taylor expansion,
\begin{align}
&[D_{Q(t + \delta )} - D_{Q(t)}]h(x) =\nonumber\\
&=\bigg{(} f_v(Q(t+ \delta,x),  Q_x(t + \delta, x) ) - f_v(Q(t,x), Q_x(t, x) ) \bigg{)} h(x) \notag \\
&\quad+ \sum_{i=1}^n \bigg{(}  f_{w_i}(Q(t+ \delta,x), Q_x(t + \delta, x) )  -f_{w_i}(Q(t,x), Q_x(t, x) ) \bigg{)} \frac{\partial h}{\partial x_i}(x) \notag \\
&= \bigg{(} f_{vv}, f_{vw} \bigg{)}(z^{\ast}(x)) \cdot \bigg{(} Q(t+ \delta,x) - Q(t,x),  Q_x(t + \delta, x) - Q_x(t , x) \bigg{)} h(x) \notag \\
&\quad+ \sum_{i=1}^n \bigg{(} f_{w_i v}, f_{w_i w} \bigg{)}(m^{\ast}(x)) \cdot \bigg{(} Q(t+ \delta,x) - Q(t,x),  Q_x(t + \delta, x) - Q_x(t , x) \bigg{)} \frac{\partial h}{\partial x_i}(x), \notag 
%+ \sum_( Q(t + \delta, x ) - Q(t,x) ) ( f_{vv}(z^{\ast}(x)) + \sum_{i=1}^n f_{w_i v}(m^{\ast}_i (x)) \frac{\partial}{\partial x_i}  )  +\sum_{i=1}^n ( Q_{x_i}(t + \delta, x ) - Q_{x_i}(t,x) ) ( f_{vv}(z^{\ast}(x))
\end{align}
where $z^{\ast}(x), m^{\ast}(x)$ are points on the line segment between $(Q(t,x), Q_x(t,x))$ and $(Q(t+ \delta,x), Q_x(t+ \delta,x))$. Due to Assumption \ref{PDEassumptions}, the derivatives of $f$ are uniformly bounded. Therefore,
\begin{align}
|( D_{Q(t + \delta )} h(x) - D_{Q(t)} h(x) ) | &\leq  C \big{(}  | Q(t+ \delta,x) - Q(t,x) | \notag \\
&\quad+ \sum_{i=1}^n | Q_{x_i}(t+ \delta,x) - Q_{x_i}(t,x) | \big{)} \times \sup_i |h(x) + \frac{\partial h}{\partial x_i}(x) |. \notag 
\end{align}

Consequently, due to Assumption \ref{PDEassumptions} and Assumption \ref{NeuralNetworkAssumptions}, 
\begin{align}
|S_{t+\delta}(x,y) - S_t(x,y)| &\leq C \big{(} |Q(t + \delta, x) - Q(t,x) | + |Q(t + \delta, y) - Q(t,y) |  \notag \\
&+ \sum_{i=1}^n (|Q_{x_i}(t + \delta, x) - Q_{x_i}(t,x) | + |Q_{y_i}(t + \delta, y) - Q_{y_i}(t,y) | ) \big{)}. \notag 
\end{align}

Therefore, we have that
\begin{align}
&\norm{ (S_{t+ \delta} - S_t) R(t + \delta)}_{L^2}^2 \leq \norm{  \int_{\Omega} ( S_{t + \delta}(\cdot, y) -  S_{t }(\cdot, y) ) R(t + \delta,y ) \mu(dy) }_{L^2}^2 \notag \\
&\qquad= \int_{\Omega}  \bigg{(}  \int_{\Omega} ( S_{t + \delta}(x, y) -  S_{t }(x, y) ) R(t + \delta,y ) \mu(dy) \bigg{)}^2 \mu(dx) \notag \\
&\qquad\leq \int_{\Omega}  \bigg{(}  \int_{\Omega} ( S_{t + \delta}(x, y) -  S_{t }(x, y) )^2 \mu(dy) \int_{\Omega} R(t + \delta,y )^2 \mu(dy) \bigg{)} \mu(dx) \notag \\
&\qquad\leq C  \norm{  Q(t + \delta) - Q(t) }_{H^1}^2,  \notag 
\end{align}
which proves the bound
\begin{align}
\langle R(t + \delta), (S_{t+ \delta} - S_t) R(t + \delta) \rangle^2 &\leq C \norm{  Q(t + \delta) - Q(t) }_{H^1}^2.\notag
\end{align}

Combining the above bounds yields
\begin{align}
| \langle R(t + \delta), S_{t+ \delta} R(t + \delta) \rangle - \langle R(t ), S_{t} R(t ) \rangle |^2 \leq C \norm{  Q(t + \delta) - Q(t) }_{H^2}^2.\notag
\end{align}
 
Furthermore, due to Lemma \ref{QTimederivativeUniformBound}, $\left|\frac{\partial Q}{\partial t}(t,x)\right|, \left|\partial_{t,x_i} Q(t,x)\right|,$ and $\left|\partial_{t, x_i x_j} Q(t,x)\right|$ are uniformly bounded in both space $x \in \Omega$ and time $t \geq 0$. Therefore, we have that
\begin{align}
|Q(t + \delta,x) - Q(t,x) | &\leq \int_t^{t + \delta}  | \frac{d Q}{d t}(s,x)| ds \leq C \delta, \notag \\
|\partial_{x_i} Q(t + \delta,x) - \partial_{x_i} Q(t,x) | &\leq \int_t^{t + \delta}  | \frac{\partial (\partial_{x_i} Q)}{\partial t}(s,x)| ds \leq C \delta, \notag \\
|\partial_{x_i x_j} Q(t + \delta,x) - \partial_{x_i x_j} Q(t,x) | &\leq \int_t^{t + \delta}  | \frac{\partial (\partial_{x_i x_j} Q)}{\partial t}(s,x)| ds \leq C \delta.\notag
\end{align}

Therefore, $|\frac{d J}{dt}|$ is in fact uniformly continuous for all $t \geq 0$. 

Since $J(t)$ is monotone decreasing (see Lemma \ref{RuniformBound}) and $J(t) \geq 0$, the monotone convergence theorem indicates that $\lim_{t \rightarrow \infty} J(t)$ exists (and is finite).

Therefore, the conditions for Barbalat's Lemma (see page 205, Section 5.5 of \cite{NonlinearSystems}) are satisfied, proving the statement of this lemma.
\end{proof}

\begin{lemma}
For any sequence $t_n \rightarrow \infty$, there is a subsequence $t_{n_{k_j}}$ such that $R(t_{n_{k_j}}) \overset{w} \rightarrow R$ in $L^2$ and $\norm{Q(t_{n_{k_j}}) - Q }_{H^1} \rightarrow 0$ (where $R$ and $Q$ may depend on the subsequence). Consequently,
\begin{eqnarray}
 \langle R(t_{n_{k_j}} ), S_{ Q(t_{n_{k_j}}) } R(t_{n_{k_j}}) \rangle \rightarrow \langle R, S_Q R \rangle,
\end{eqnarray}
where $S_Q$ is
\begin{eqnarray}
S_Q(x,y) = \bigg{\langle} D_{Q} \bigg{[}    \nabla_{\theta} \big{(} \eta(x) c \sigma (w x + b) \big{)} \bigg{]} \cdot  D_Q \bigg{[}   \nabla_{\theta} \big{(} \eta(y) c \sigma (w y + b) \big{)}  \bigg{]}  , \mu_0 \bigg{\rangle}. 
\end{eqnarray}

\end{lemma}
\begin{proof}
By Lemma \ref{WeaklyConvergentSubseq}, there is a subsequence $t_{n_k}$ such that $R_{t_{n_{k}}} \overset{w} \rightarrow R$ for some $R \in L^2$ (where $R$ may depend upon the sequence $t_{n_k}$). By Lemma \ref{StronglyConvergentSubseq}, there is a further sub-subsequence $t_{n_{k_j}}$ where $\norm{Q(t_{n_{k_j}}) - Q }_{H^1} \rightarrow 0$ for some $Q \in H^1$ (where $Q$ may depend upon the sequence $t_{n_{k_j}}$). For notational convenience, define $t_i = t_{n_{k_j}}$. We first establish the strong convergence of $S_{t_i} = S_{Q(t_i)}$. 
\begin{align}
& S_Q(x,y) - S_{Q(t_i)}(x,y)  =\nonumber\\
&=\bigg{\langle} D_Q [  \nabla_{\theta} ( \eta(x) c \sigma (w x + b) ) ] \cdot  D_Q[ \nabla_{\theta} ( \eta(y) c \sigma (w y + b) ) ] , \mu_0 \bigg{\rangle} \notag \\
&\quad-  \bigg{\langle} D_{Q(t_i)}  [  \nabla_{\theta} ( \eta(x) c \sigma (w x + b) )  ]\cdot  D_{Q(t_i)} [ \nabla_{\theta} ( \eta(y) c \sigma (w y + b) ) ], \mu_0 \bigg{\rangle} \notag \\
&= \bigg{\langle} \big{(} D_{Q} - D_{Q(t_i)}  \big{)} [ \nabla_{\theta}  ( \eta(x) c \sigma (w x + b) )  ] \cdot   D_{Q} [ \nabla_{\theta} ( \eta(y) c \sigma (w y + b) ) ], \mu_0 \bigg{\rangle} \notag \\
&\quad+  \bigg{\langle} D_{Q(t_i)} [ \nabla_{\theta} ( \eta(x) c \sigma (w x + b) ) ] \cdot  \big{(} D_Q - D_{Q(t_i)}  \big{)} [ \nabla_{\theta} ( \eta(y) c \sigma (w y + b) ) ] , \mu_0 \bigg{\rangle}.\nonumber
\end{align}

Consequently, due to Assumption \ref{PDEassumptions} and Assumption \ref{NeuralNetworkAssumptions},
\begin{align}
&| S_{Q(t_i)}(x,y) - S_Q(x,y) | \leq\nonumber\\
&\leq  C_1 \times \sup_{x \in \Omega} \bigg{[}  |f_{vv}(x)| +  \sum_{j=1}^n |f_{v w_j}(x) |  +  \sum_{j,m =1}^n |f_{w_m w_j}(x) | \bigg{]} \notag \\
&\qquad\times \bigg{(} | Q(t_i,x ) - Q(x) | + | Q(t_i,y ) - Q(y) | \notag \\
&\quad+ \sum_{j=1}^n \big{(} | Q_{x_j}(t_i,x ) - Q_{x_j}(x) | + | Q_{x_j}(t_i,y ) - Q_{x_j}(y) |  \big{)} \bigg{)}. 
\label{Sinequality}
\end{align}

Similarly,  due to Assumption \ref{PDEassumptions} and Assumption \ref{NeuralNetworkAssumptions}, we have the uniform bound
\begin{align}
| S_Q(x,y) | &\leq C.\nonumber 
\end{align}

Now, let us next consider the decomposition
\begin{align}
& \langle R(t_i ), S_{Q( t_i ) } R(t_i ) \rangle =  \langle R(t_i ), (S_{Q( t_i ) } - S_Q + S_Q ) R(t_i ) \rangle  \notag \\
 &\qquad=  \langle R(t_i ), S_Q  R(t_i ) \rangle + \langle R(t_i ), (S_{ Q(t_i) } - S_Q )  R(t_i ) \rangle \notag \\
 &\qquad=  \langle R(t_i ), S_Q  R \rangle + \langle R(t_i ), S_Q  (R(t_i ) - R ) \rangle + \langle R(t_i ), (S_{ Q(t_i) } - S_Q )  R(t_i ) \rangle \notag
\end{align}

Let us begin by proving the function $G = S_Q R$ is in $L^2$, where $G(x) = \int_{\Omega} S_Q(x,y) R(y) \mu(dy)$. Using the Cauchy-Schwarz inequality,
\begin{align}
\int_{\Omega} G(x)^2 \mu(dx) &= \int_{\Omega} \bigg{(} \int_{\Omega} S_Q(x,y) R(y) \mu(dy) \bigg{)}^2 \mu(dx) \notag \\
&\leq \int_{\Omega}  \bigg{(}  \int_{\Omega} S_Q(x,y)^2  \mu(dy) \times   \int_{\Omega} R(y)^2  \mu(dy)  \bigg{)} \mu(dx)\notag \\
&\leq C. 
\end{align}

Therefore, since $S_Q(x,y)$ is uniformly bounded and $R \in L^2$, $S_Q R \in L^2$. By the definition of weak convergence,
\begin{eqnarray}
\langle R(t_i ), S_Q  R \rangle  \rightarrow \langle R, S_Q  R \rangle.\notag 
\end{eqnarray}

The next term we consider is $\langle R(t_i ), (S_{ Q(t_i) } - S_Q )  R(t_i ) \rangle$. Define the function 
\begin{align}
H(t, x) &= \int_{\Omega} (S_{ Q(t) }(x,y) - S_Q(x,y) )  R(t ,y) \mu(dy).
\end{align}

Then, using the inequality (\ref{Sinequality}) and the Cauchy-Schwarz inequality, 
\begin{align}
&\int_{\Omega} H(t_i, x)^2 \mu(dx) = \int_{\Omega} \bigg{(}  \int_{\Omega} (S_{ Q(t_i) }(x,y) - S_Q(x,y) )  R(t_i,y) \mu(dy) \bigg{)}^2 \mu(dx) \notag \\
&\leq C \int_{\Omega} \bigg{(}  \int_{\Omega} \bigg{(} | Q(t_i,x ) - Q(x) | +  | Q(t_i,y ) - Q(y) | \notag \\
&\quad+ \sum_{j=1}^n \big{(} | Q_{x_j}(t_i,x ) - Q_{x_j}(x) | +  | Q_{x_j}(t_i,y ) - Q_{x_j}(y) | \big{)} \bigg{)}     \times | R(t_i ,y)  | \mu(dy) \bigg{)}^2 \mu(dx) \notag \\
&\leq C \int_{\Omega}   \int_{\Omega} \bigg{(} | Q(t_i,x ) - Q(x) | +  | Q(t_i,y ) - Q(y) | \notag \\
&\quad+ \sum_{j=1}^n \big{(} | Q_{x_j}(t_i,x ) - Q_{x_j}(x) | +  | Q_{x_j}(t_i,y ) - Q_{x_j}(y) | \big{)} \bigg{)}^2 \mu(dy) \times \int_{\Omega} | R(t_i ,y)  |^2 \mu(dy)  \mu(dx) \notag \\
&\leq C \int_{\Omega}  \int_{\Omega} \bigg{(} | Q(t_i,x ) - Q(x) |^2 +  | Q(t_i,y ) - Q(y) |^2 \notag \\
&\quad+ \sum_{j=1}^n \big{(} | Q_{x_j}(t_i,x ) - Q_{x_j}(x) |^2 +  | Q_{x_j}(t_i,y ) - Q_{x_j}(y) |^2 \big{)} \bigg{)} \mu(dy) \times \int_{\Omega} | R(t_i ,y)  |^2 \mu(dy)  \mu(dx) \notag \\
%&\leq& C \int_{\Omega} \norm{  Q(t_i ) - Q }_{L^2}^2  \times \norm{ R(t_i)  }_{L^2}^2  dx + C \int_{\Omega} | Q(t_i,x ) - Q(x) |^2  \times \norm{ R(t_i)  }_{L^2}^2  dx  \notag \\
%&\leq& C \norm{  Q(t_i ) - Q }_{L^2} ,
&\leq C \norm{  Q(t_i ) - Q }_{H^1}^2 \times  \norm{ R(t_i) }_{L^2}^2 \notag \\
&\leq C \norm{  Q(t_i ) - Q }_{H^1}^2,\notag
\end{align}
where we have used the uniform bound for $R(t)$ from Lemma \ref{RuniformBound} in the last line. Consequently, due to the strong convergence of $Q(t_i )$ in $H^1$,
\begin{eqnarray}
\norm{ H(t_i) }_{L^2} \rightarrow 0. \notag
\end{eqnarray}
This allows us to conclude 
\begin{eqnarray}
\langle R(t_i ), (S_{Q( t_i) } - S_Q )  R(t_i ) \rangle^2 &=& \langle R(t_i), H(t_i) \rangle^2 \notag \\
&\leq& \norm{ R(t_i)}_{L^2}^2 \norm{ H(t_i)}_{L^2}^2 \rightarrow 0,\notag
\end{eqnarray}
since $\norm{ R(t)}_{L^2}^2$ is uniformly bounded.

The third term to be analyzed is $\langle R(t_i ), S_Q  (R(t_i ) - R ) \rangle$. Define $V(t,x) $ by
\begin{eqnarray}
V(t,x) = \int_{\Omega} S_Q(x,y) (R(t, y) - R(y) ) \mu(dy).\notag
\end{eqnarray}
For each $x \in \Omega$, $S_Q(x,\cdot) \in L^2$. Therefore, for each $x \in \Omega$, $V(t_i,x) \rightarrow 0$ since $R(t_i)$ weakly converges to $R$ in $L^2$. Then,
\begin{eqnarray}
|\langle R(t_i ), S_Q  (R(t_i ) - R ) \rangle |^2 &\leq& \int_{\Omega} R(t_i,x)^2 \mu(dx) \times \int_{\Omega} V(t_i,x)^2 \mu(dx). 
\label{BoundLastTerm}
\end{eqnarray}

Using Young's inequality and the previous bounds, 
\begin{align}
V(t,x)^2 &\leq \frac{1}{2} \int_{\Omega} S_Q(x,y)^2 \mu(dy) + \frac{1}{2} \int_{\Omega} (R(t, y) - R(y) )^2 \mu(dy) \notag \\
&\leq  \frac{1}{2} \int_{\Omega} S_Q(x,y)^2 \mu(dy) + \int_{\Omega} R(t, y)^2 \mu(dy) + \int_{\Omega} R(y)^2 \mu(dy) \notag \\
&\leq C < \infty.\notag
\end{align}

Therefore, $V(t,x)$ is bounded by an integrable function on $\Omega$ and the dominated convergence theorem yields 
\begin{eqnarray}
\int_{\Omega} V(t_i,x)^2 \mu(dx) \rightarrow 0.
\end{eqnarray}
Combining this with the bound in equation (\ref{BoundLastTerm}), $|\langle R(t_i ), S_Q  (R(t_i ) - R ) \rangle |^2 \rightarrow 0$. 

\end{proof}

We can now prove that the PDE residual of the neural network converges to zero in the weak sense as training time $t \rightarrow \infty.$

\begin{theorem}
\label{MainTheoremPDEResidualGlobalConvergence}
The PDE residual $R(t,x) = \mathcal{A}_{Q(t)} Q(t,x) - g(x)$ of the neural network $Q(t,x)$ converges to zero in the weak sense as training time $t \rightarrow \infty$. In particular, for any sequence $t_{n}\rightarrow\infty$, there exists a subsequence $t_{n_{k}}\rightarrow\infty$ such that
\begin{eqnarray}
\int_{\Omega} \phi(x) R(t_{n_{k} },x) dx \rightarrow 0,
\end{eqnarray}
for all $\phi \in H_0^1$.
\end{theorem}
\begin{proof}
For any sequence $t_n \rightarrow \infty$, there exists a subsequence $t_{n_{k_j}}$ such that
\begin{eqnarray}
\langle R(t_{n_{k_j}} ), S_{ Q(t_{n_{k_j}} ) } R(t_{n_{k_j}}) \rangle \rightarrow \langle R, S_Q R \rangle.
\end{eqnarray}
Due to Lemma \ref{BarbalatLemma}, $\lim_{t \rightarrow \infty} \langle R(t), S_{Q(t)} R(t) \rangle = 0$. Therefore,  $\langle R(t_{n_{k_j}} ), S_{ Q(t_{n_{k_j}}) } R(t_{n_{k_j}}) \rangle \rightarrow 0$. The limit point $\langle R, S_Q R \rangle$ (which may depend upon the specific subsequence $t_{n_{k_j}}$) must then also satisfy
\begin{eqnarray}
\langle R, S_Q R \rangle = 0. \notag
\end{eqnarray}

The fact that $H_0^1$ is closed gives then, by Theorem \ref{GlobalMinimizerFixedPoint}, the statement of the Theorem.
\end{proof}

We can now prove that the limit neural network converges to the solution of the PDE (\ref{PDE0}) solution as training time $t \rightarrow \infty.$

\begin{theorem} \label{MainTheoremPDESolutionGlobalConvergence}
The neural network function $Q(t,x)$  converges to the solution $u(x)$ of the PDE (\ref{PDE0}) in $H^1$ as $t \rightarrow \infty$:
\begin{eqnarray}
\lim_{t \rightarrow \infty} \norm{ Q(t) - u }_{H^1} = 0.
\end{eqnarray}
\end{theorem}
\begin{proof}

Recall from Lemma \ref{Qbound} that $\norm{ Q(t) }_{H^2}$ is uniformly bounded. Then, for any sequence of times $t_n \rightarrow \infty$, there exists a subsequence $Q(t_{n_{k}}) \in H^1$ which converges to a limit point $Q$ (which may depend upon the subsequence). Due to Lemma \ref{Hestimate0apriori} and Lemma \ref{Hestimate0}, each $Q(t_{n_{k}})$ furthermore takes values in $H_0^1$. $H_0^1$ is a closed subset of $H^1$.  Therefore, $Q(t_{n_{k}})$ (strongly) converges to a limit point in $H_0^1$.

From Theorem \ref{MainTheoremPDEResidualGlobalConvergence}, we know that for any $\phi \in H_0^1$,
\begin{eqnarray}
\int_{\Omega} \phi(x) \bigg{(} \mathcal{A}_{   Q(t_{n_{k} } )}  Q(t_{n_{k}},x )- g(x) \bigg{)} dx \rightarrow 0. 
\end{eqnarray}
For notational convenience, define the sequence $Q(t_{n_{k} } )$ as $Q(t_{\ell} )$, where $\ell = n_k$. Integrating by parts yields
\begin{align}
&  \int_{\Omega} \bigg{(} \sum_{i,j=1}^n a_{ij}(x) \frac{\partial \phi}{\partial x_i}(x)  \frac{\partial Q }{\partial x_j}(t_{\ell}, x)  + \sum_{i=1}^n c_i(x) \frac{\partial Q}{\partial x_i}(t_{\ell}, x)  \phi(x)  - d(x) Q(t_{\ell}, x) \phi(x)  \notag \\
&\quad+ f( Q(t_{\ell}, x),  Q_x(t_{\ell}, x) ) \phi(x) - g(x) \phi(x) \bigg{)} dx \rightarrow 0. \notag
\end{align}

Since $Q(t_{\ell} )$ converges strongly in $H^1_0$, $f$ and its derivatives are uniformly bounded, and $\phi, a_{ij}, c_i, d$ are smooth functions on a compact domain, the limit point satisfies
\begin{align}
&  \int_{\Omega} \bigg{(} \sum_{i,j=1}^n a_{ij}(x) \frac{\partial \phi}{\partial x_i}(x)  \frac{\partial Q }{\partial x_j}(x)  + \sum_{i=1}^n c_i(x) \frac{\partial Q}{\partial x_i}( x)  \phi(x)  - d(x) Q( x) \phi(x)  \notag \\
&\quad+ f( Q( x), Q_x(x) ) \phi(x) - g(x) \phi(x) \bigg{)} dx = 0. 
\label{WeakFormOfPDELimitPoint}
\end{align}

Since $H_0^1$ is a closed space, the limit point $Q$ of the subsequence $t_{n_{k}}$ is an element of $H_0^1$. Note that equation (\ref{WeakFormOfPDELimitPoint}) holds for all $\phi \in H_0^1$. Due to Assumption \ref{PDEassumptions2} (as discussed in Corollary \ref{UniqueWeakSolution}), $Q$ is then the unique weak solution of the original PDE (\ref{PDE0}). Therefore, $\norm{ Q(t_{n_k}) - u }_{H^1} \rightarrow 0$, where $u$ is the unique weak solution to (\ref{PDE0}). Since the sequence $t_n$ is arbitrary, $\lim_{t \rightarrow \infty} \norm{ Q(t) - u }_{H^1} = 0$.  

\end{proof}

We can now prove that the PDE residual of the neural network converges to zero as training time $t \rightarrow \infty$ in the strong sense.
\begin{theorem}
\label{MainTheoremPDEResidualGlobalConvergenceStrong}
Let Assumptions \ref{PDEassumptions} through \ref{phi} hold. Then, the PDE residual $R(t,x) = \mathcal{A}_{Q(t)} Q(t,x) - g(x)$ of the neural network $Q(t,x)$ converges to zero in $L^2$ as training time $t \rightarrow \infty$. In particular, this also means that the objective function $J(t) \rightarrow 0$ and the neural network converges to a global minimizer.
\begin{eqnarray}
\lim_{t \rightarrow \infty} J(t) = \lim_{t \rightarrow \infty} \frac{1}{2} \norm{R(t) }_{L^2}^2 = 0.
\end{eqnarray}
\end{theorem}

\begin{proof}[Proof of Theorem \ref{MainTheoremPDEResidualGlobalConvergenceStrong}]
First we establish notation and some critical definitions. Let us recall that $\theta=(c,w,b)\in \mathbb{R}^{n+2}$ and set $\psi_{\theta}(x)=\nabla_{\theta}\left[\eta(x)c\sigma(wx+b)\right]$. Then we define the bounded linear map
\begin{align*}
M: (L^{2}(\mu_{0}) )^{n+2} \mapsto H_{(0)}^2, \text{ such that } \left[M\zeta\right](x)=\left<\psi_{\theta}(x)^{\top} \zeta(\theta),\mu_{0}\right>,
\end{align*}
which is well defined since $\left<\|\psi_{\theta}\|^{2}_{H^{2}},\mu_{0}\right><\infty$. Next we set
\begin{align*}
\xi(t,\theta)=\left<D_{Q(t)}\psi_{\theta}, R(t)\right>_{L^{2}}\in \mathbb{R}^{n+2}.
\end{align*}

It follows from our assumptions that $\sup_{t<\infty}\|\xi(t)\|_{L^{2}(\mu_{0})}<\infty$ and that $t\mapsto \xi(t)$ is continuous in $L^{2}(\mu_{0})$.

The proof now builds on a number of smaller steps.

\vspace{0.2cm}

\underline{Step 1:} The limit network $Q(t,x)$ can be written as follows
\begin{align*}
Q(t,x)&=\left[M\zeta(t)\right](x), \text{ with }\zeta(t,\theta)=-\alpha\int_{0}^{t}\xi(s,\theta)ds.
\end{align*}

This follows by Fubini Theorem. Indeed, recalling the definition of $\psi_{\theta}(x)$ we have
\begin{align*}
\int_{\Omega} U_{Q(t)}(x,y)R(t,y)\mu(dy)&=\left<\psi_{\theta}(x)\xi(t,\theta),\mu_{0}\right>=\left[M\xi(t)\right](x).
\end{align*}

Hence, we can write that $\frac{\partial Q(t,x)}{dt}=-\alpha \left[M\xi(t)\right](x)$. Since $Q(0)=0$ and $M$ is a bounded linear operator, we can integrate to get the statement of Step 1.

\vspace{0.2cm}

\underline{Step 2:} Fix an arbitrary $\hat{\zeta}\in L^{2}(\mu_{0})$ and set $\hat{u}=\left[M\hat{\zeta}\right]\in H_{(0)}^2$. Then, if we define $N(t)=\frac{1}{2}\|\zeta(t)-\hat{\zeta}\|^{2}_{L^{2}(\mu_{0})}$, we get that
\begin{align*}
\frac{d N(t)}{dt}&=\left<\zeta(t)-\hat{\zeta},\frac{d}{dt}\zeta(t)\right>_{L^{2}(\mu_{0})}\nonumber\\
&=-\alpha \left<(\zeta(t)-\hat{\zeta})(\theta)\left<D_{Q(t)}\psi_{\theta}, R(t)\right>,\mu_{0}\right>\nonumber\\
&=-\alpha\left<D_{Q}\left<\psi_{\theta}(\zeta(t)-\hat{\zeta}),\mu_{0}\right>,R(t)\right>_{L^{2}}\nonumber\\
&=-\alpha\left<D_{Q}(Q(t)-\hat{u}),R(t)\right>_{L^{2}}.
\end{align*}
 
 \vspace{0.2cm}

\underline{Step 3:} Ignoring for notational convenience the time dependence, we can write 
\begin{align*}
D_{Q}(Q-u)&= R+R^{\circ}, 
\end{align*}
with $|R^{\circ}|\leq \|\nabla f\|_{\infty}\left(|Q-u|+|\nabla (Q-u)|\right)$.

This claim follows by some algebra rearrangement. Indeed, first we recall that $u$ satisfies the PDE $\mathcal{L}u+f(u,u_{x})=g$.  Hence, if we set $e=Q-u$ we have
\begin{align*}
R&=\mathcal{L}Q+f(Q,Q_{x})-g=\mathcal{L}e+f(Q,Q_{x})-f(u,u_{x}),
\end{align*}
and
\begin{align*}
D_{Q}e&=\mathcal{L}e+f_{u}(Q,Q_{x})e+f_{w}(Q,Q_{x})e_{x}.
\end{align*}

Taylor expansion gives
\begin{align*}
f(u,u_{x})&=f(Q,Q_{x})-f_{u}(Q,Q_{x})e-f_{w}(Q,Q_{x})e_{x}+R^{\circ}.
\end{align*}

Hence, we indeed obtain that
\begin{align*}
D_{Q}(Q-u)-R&=f_{u}(Q,Q_{x})e+f_{w}(Q,Q_{x})e_{x}-\left(f(Q,Q_{x})-f(u,u_{x})\right)=R^{\circ},
\end{align*}
establishing the statement of this step.

 \vspace{0.2cm}

\underline{Step 4:} Let us define $m(t)=\|Q-u\|_{L^{2}}+\|\nabla(Q-u)\|_{L^{2}}$ and set the constant $C_{f}=\|\nabla f\|_{\infty}$. Then, we can write
\begin{align*}
\frac{d N(t)}{dt}&\leq -\alpha J(t)+\alpha C_{f}^{2}|m(t)|^{2}+\alpha C\|u-\hat{u}\|_{H^{2}}^{2}.
\end{align*} 
for some $0<C<\infty$.

Indeed, by Steps 2 and 3, we can write 
\begin{align*}
\frac{d N(t)}{dt}&=-\alpha\|R(t)\|^{2}-\alpha\left<R^{\circ}(t),R(t)\right>-\alpha\left<D_{Q(t)}(u-\hat{u}),R(t)\right>.
\end{align*} 
 
 Then, by Step 3, the Cauchy-Schwartz inequality and then the Young's inequality we have
(ignoring the time dependence)
\begin{align*}
 |\left<R,R^{\circ}\right>|&\leq C_{f}\left\||Q-u|+|\nabla(Q-u)|\right\|_{L^{2}}\|R\|_{L^{2}}\nonumber\\
 &\leq C_{f}m \|R\|_{L^{2}}\nonumber\\
 &\leq C^{2}_{f}m^{2}+\frac{1}{4}\|R\|^{2}_{L^{2}}.
 \end{align*}

Similarly we also have for some $0<C<\infty$
\begin{align*}
|\left<D_{Q}(u-\hat{u}),R\right>|&\leq C^{2}\|u-\hat{u}\|_{H^{2}}^{2}+\frac{1}{4}\|R\|^{2}_{L^{2}}.
\end{align*}
We note that the PDE solution $u \in H_{(0)}^2$ due to Assumptions \ref{PDEassumptions} and \ref{PDEassumptions2}.

Combining these last two expressions and the identity $J(t)=\frac{1}{2}\|R(t)\|^{2}_{L^{2}}$ we obtain the statement of Step 4.

\vspace{0.2cm}

\underline{Step 5:} For any $0<T<\infty$ and for any $\hat{\zeta}\in L^{2}(\mu_{0})$, we have that 
\begin{align}
J(T)\leq \frac{\|\hat{\zeta}\|^{2}_{L^{2}(\mu_{0})}}{2\alpha T}+C_{f}^{2}\frac{1}{T}\int_{0}^{T}|m(s)|^{2}ds+C^{2}\|u-[M\hat{\zeta}]\|^{2}_{H^{2}}.\label{Eq:CriticalBoundJ}
\end{align}

Indeed, we integrate the expression in Step 4, use the fact that $N(t)\geq 0$ and $N(0)=\frac{1}{2}\|\hat{\zeta}\|^{2}_{L^{2}(\mu_{0})}$ (since $\zeta(0)=0$), and we get
\begin{align*}
\int_{0}^{T}J(s)ds &\leq \frac{\|\hat{\zeta}\|^{2}_{L^{2}(\mu_{0})}}{2\alpha}+C_{f}^{2}\int_{0}^{T}|m(s)|^{2}ds+T C^{2}\|u-[M\hat{\zeta}]\|^{2}_{H^{2}}.
\end{align*}

Step 5 then follows by the observation that since $T\mapsto J(T)$ is non-increasing, we have, for every $T<\infty$, that $TJ(T)\leq \int_{0}^{T}J(s)ds$ and then dividing the resulting expressions by $T$.

Having established now (\ref{Eq:CriticalBoundJ}) the theorem follows directly. Indeed, for any fixed $\hat{\zeta}\in L^{2}(\mu_{0})$, we can get the limit as $T\rightarrow\infty$ to (\ref{Eq:CriticalBoundJ}). The first term on the left hand side of (\ref{Eq:CriticalBoundJ}) goes to zero for any fixed $\hat{\zeta}$ as $T\rightarrow\infty$. The second term, also goes to zero as $T\rightarrow\infty$, which follows by Theorem \ref{MainTheoremPDESolutionGlobalConvergence} and Lemma \ref{Qbound} together with elliptic regularity, which give that $\lim_{t\rightarrow\infty}m(t)=0$ and $\sup_{t}m(t)<\infty$ respectively. Thus, we shall have that
\begin{align*}
\lim_{T\rightarrow\infty} J(T)\leq C^{2}\|u-[M\hat{\zeta}]\|^{2}_{H^{2}}.
\end{align*}

The latter expression holds for all time-independent $\hat{\zeta}\in L^{2}(\mu_{0})$.

\vspace{0.2cm}

\underline{Step 6:} Last step is to show that for any $\epsilon>0$ we can choose an element $\hat{\zeta}\in L^{2}(\mu_{0})$, so that $\|u-[M\hat{\zeta}]\|^{2}_{H^{2}}\leq \epsilon$.

To show this we shall use the density result of \cite{QPDE} asserting that the linear span of $\{ \eta(x) \sigma( w \cdot x + b ) \}_{w,b \in \mathbb{R}^{n+1} }$ is dense in $H^2_{(0)}$. This implies that it is enough to approximate a given function of the form $\psi^{*}_{1,\theta}(x)=\eta(x)\sigma(w^{*}x+b^{*})$. If we denote by $B_{\rho}$ the ball in $\mathbb{R}^{n+1}$ with radius $\rho>0$, the assumption of full support of $\mu_{0}$ gives that $\mu_{0,\rho}=\mu_{0}\left\{(w - w^{\ast},b - b^{\ast} )\in B_{\rho}\right\}>0$. Then, we focus on the first element of the vector $\theta=(c,w,b)$ and we define $\hat{\zeta}_{\rho}(\theta)=\left\{\frac{1_{B_{\rho}}(w - w^{\ast},b - b^{\ast} )}{\mu_{0,\rho}},0,\cdots,0\right\}\in L^{2}(\mu_{0})$ with $\|\hat{\zeta}_{\rho}\|^{2}_{L^{2}}=\frac{1}{\mu_{0,\rho}}$. Subsequently, this then yields
\begin{align*}
[M\hat{\zeta}](x)&=\frac{\eta(x)}{\mu_{0,\rho}}\left<\sigma(wx+b)1_{B_{\rho}}(w -w^{\ast},b-b^{\ast}),\mu_{0}\right>.
\end{align*}
This then gives
\begin{align*}
\left\|M\hat{\zeta}-\psi^{*}_{1}\right\|_{C^{2}(\bar{\Omega})}&\leq \sup_{(w,b)\in B_{\rho}}\left\|\eta(x)\sigma(wx+b)-\eta(x)\sigma(w^{*}x+b^{*})\right\|_{C^{2}(\bar{\Omega})}\nonumber\\
&\leq C(1+\|w^{*}\|+\rho)^{3}\rho\nonumber\\
&\rightarrow 0, \text{ as }\rho\rightarrow 0.
\end{align*}

However, on bounded domains we have that $\|\cdot\|_{H^{2}(\bar{\Omega})}\leq C\|\cdot\|_{C^{2}(\bar{\Omega})}$. Subsequently, we finally obtain that for any $\epsilon>0$ we can choose an element $\hat{\zeta}\in L^{2}(\mu_{0})$, so that $\|u-[M\hat{\zeta}]\|^{2}_{H^{2}}\leq \epsilon$. 

Thus, we have established that for any $\epsilon>0$, 
\begin{align*}
\lim_{T\rightarrow\infty} J(T)\leq C^{2}\epsilon,
\end{align*}
yielding what we wanted.
\end{proof}

\section{Convergence of neural network to limit PDE solution as $N \rightarrow \infty$}\label{S:ConvergenceLimitPDE_LargeN}

In this section, we prove convergence of the finite-width neural network $Q^N(t,x)$ to the solution of the limit neural network PDE for $Q(t,x)$. In the calculations in this section, $C$ will be used as a generic finite positive constant which may change from line to line. 

We first establish \emph{a priori} bounds on the parameter evolution.
\begin{align}
| \theta_i(t) - \theta_i(0) | &\leq C N^{2 \beta -1 } N^{\gamma - \beta} \notag \\
&\leq C N^{\beta -1 + \gamma},
\label{NTKbound0}
\end{align}
where we have chosen $\gamma + \beta <1$. The constant $C=C_{T}<\infty$ is a finite constant depending on  $T$, with $t\in[0,T]$. (\ref{NTKbound0}) follows directly from (\ref{ObjGrad12iclipped}) and (\ref{ActualClippedGradientDescentTrainingAlgorithm}). 
(\ref{NTKbound0}), the triangle inequality, and $\beta -1 + \gamma < 0$ (from Assumption \ref{phi}) directly yield the bound 
\begin{eqnarray}
| \theta_i(t) | &\leq& | \theta_i(0) | + C N^{\beta -1 + \gamma} \notag \\
&\leq& | \theta_i(0) | + C.
\label{NTKbound1}
\end{eqnarray}
This also allows us to appropriately bound the moments of the trained parameters on $t \in [0,T]$ 
\begin{eqnarray}
\mathbb{E}[ |\theta_i(t)|^p ] &\leq& \mathbb{E} [ |\theta_i(0) + C N^{\beta -1 + \gamma} |^p ]. 
\label{TrainedParameterBound}
\end{eqnarray}
(\ref{NTKbound0}) and (\ref{TrainedParameterBound}) are element-wise bounds for the vectors $|\theta_i(t) - \theta_i(0)|$ and $|\theta_i(t)|^p$, where in these equations $|\cdot|$ and $|\cdot|^p$ are also applied element-wise.\footnote{If $z$ is a vector of length $d$, then $|z|^p = (|z_1|^p, |z_2|^p, \ldots, |z_d|^p)$. $|z| \leq C$ is defined as $|z_1| \leq C, |z_2| \leq C, \ldots, |z_d| \leq C$.}

Importantly, since the parameters $c^i_0$ are initialized as uniformly bounded random variables,
\begin{align}
 |c_t^i | &\leq  | c_0^i |  + C N^{\beta -1 + \gamma}  \notag \\
 &\leq K<\infty.
\label{TrainedParameterCBound}
\end{align}

\begin{lemma} \label{MarcinkiewiczZygmundLemma}
Define $A^N_i(x) = \partial_{x}^i [ \eta(x) N^{-\beta} \sum_{j=1}^N c^j_0 \sigma(w^j_0 \cdot x + b^j_0) ]$. Then, for $|i| \leq 2$,
\begin{eqnarray}
\mathbb{P}[ \lim_{N \rightarrow \infty} \sup_{x \in \Omega} | A_i^N(x)| = 0 ] = 1. 
\end{eqnarray}
\end{lemma}
\begin{proof}
Recall that $\theta^j_0 = (c^j_0, w^j_0, b^j_0)$ are i.i.d. RVs and $c^j_0$ is mean zero. By the Marcinkiewicz-Zygmund strong law of large numbers, $A^N_i(x) = \partial_{x}^i [\eta(x) N^{-\beta} \sum_{j=1}^N c^j_0 \sigma(w^j_0 \cdot x + b^j_0) ]$ converges almost surely to zero for each $x \in \Omega$. For a $\delta(N) > 0$, construct a finite sub-cover $U_k^N = \{ x: | x - x_k | \leq \delta(N) \}$ with $k = 1, \ldots, K(N)$ for the compact set $\Omega$. We select $\delta(N) = N^{-q}$ where $q = 2 - \beta$. Then, there exists a finite sub-cover with $K(N) = \lfloor C \delta(N)^{-n} \rfloor = \lfloor C N^{q n} \rfloor$ where $n$ is the spatial dimension of $\Omega \subset \mathbb{R}^n$. That is, for each $N \geq 1$, we construct a new finite sub-cover, where the total number of members of the sub-cover grows with $N$. 
%Note that $\{ A^N_i(x_k) \}_{i = 0,1,2; k = 0, 1, \ldots, K(N)}$ converges almost surely to zero.

Define $f_{i,j}(x) = \partial_{x}^i [ \eta(x) \sigma(w^j_0 \cdot x + b^j_0) ]$. Note that $\sup_{x \in \Omega} |\frac{\partial}{\partial x_m} [f_{i,j}(x)]| \leq C (1 + \norm{w_0^j}_1 )^3$. Define $Y_i^N(x) = A^N_i(x) - A^N_i(x_k)$. Using a Taylor expansion,
\begin{align}
Y_i^N(x) &=  N^{-\beta} \sum_{j=1}^N c^j_0  \big{(}  f_{i,j}(x) - f_{i,j}(x_k) \big{)} \notag \\
&= N^{-\beta} \sum_{j=1}^N c^j_0  \nabla_x  f_{i,j}(x^{\ast}_j) \cdot (x  - x_k ),\notag
\end{align}
where $x^{\ast}_j$ is a point on the line segment connecting $x$ and $x_k$. $\nabla_x  f_{i,j}(x^{\ast}_j)$ is polynomially bounded by $w_0^j$, $c^j_0$ is bounded, and $c^j_0$ is furthermore mean zero. Using the triangle inequality,
\begin{align}
\sup_{x \in \Omega} | A_i^N(x)| &\leq \max_{1 \leq m \leq K(N)} | A_i^N(x_m) | \notag \\
&\quad+ \sup_{x, m: | x - x_m | \leq \delta(N), m = 1, \ldots, K(N) } | Y_i^N(x)  |.
\label{TriangleAofX}
\end{align}

We first address the second term in (\ref{TriangleAofX}) using the identity $q+\beta=2$
\begin{align}
\sup_{x, m: | x - x_m | \leq \delta(N), m = 1, \ldots, K(N) } | Y_i^N(x) | &\leq C N^{-q - \beta}  \sum_{j=1}^N (1 + \norm{w_0^j }_1)^3 \notag \\
&= C N^{-1} \times \bigg{(} N^{-1} \sum_{j=1}^N (1 + \norm{w_0^j }_1)^3 \bigg{)}. \notag 
\end{align}

By the strong law of large numbers, the $N^{-1} \sum_{j=1}^N (1 + \norm{w_0^j }_1)^3$ converges a.s. to a finite constant. Therefore, due to the additional factor of $N^{-1}$, 
\begin{eqnarray}
\mathbb{P} \bigg{[} \lim_{N \rightarrow \infty} \sup_{x, m: | x - x_m | \leq \delta(N), m = 1, \ldots, K(N) } | Y_i^N(x) | = 0 \bigg{]} = 1.
\label{ASconvergenceSecondTermY}
\end{eqnarray}

We now consider the first term in (\ref{TriangleAofX}) using the union bound for probabilities. For an (arbitrary) $\epsilon > 0$, define the event $G^N = \{ \max_{1 \leq m \leq K(N)} | A_i^N(x_m) | \geq \epsilon \}$. Then, 
\begin{align}
\mathbb{P}[ G^N ] &\leq \sum_{m=1}^{K(N)}  \mathbb{P}[ | A_i^N(x_m) |  \geq \epsilon].
\label{UnionBound2}
\end{align}

Define the i.i.d. mean-zero random variables $W_j = \partial_{x}^i [\eta(x_m)  c^j_0 \sigma(w^j_0 \cdot x_m + b^j_0)]$. For a positive integer $p> \frac{2 + (2 - \beta)n}{2 (2 \beta -1) }$, we can use Markov's inequality, the Marcinkiewicz-Zygmund inequality, and Jensen's inequality to bound the probability terms in the sum in (\ref{UnionBound2}):
\begin{align}
\mathbb{P}[ | A_i^N(x_m) |  \geq \epsilon ] &= \mathbb{P}[ | A_i^N(x_m) |^{4p}  \geq \epsilon^{4p}] \notag \\
&\leq \epsilon^{-4p} \times \mathbb{E}[ | A_i^N(x_m) |^{4p} ] \notag \\
&\leq \epsilon^{-4p} N^{-\beta \times 4p} \mathbb{E}\big{[}  |\sum_{j=1}^N W_j|^{4p} \big{]} \notag \\
&\leq C \epsilon^{-4p} N^{-\beta \times 4p} \mathbb{E}\big{[}  \big{(} \sum_{j=1}^N | W_j|^2 \big{)} ^{2p} \big{]} \notag \\
&= C \epsilon^{-4p} N^{-\beta \times 4p} N^{2p} \mathbb{E}\big{[}  \big{(} \frac{1}{N} \sum_{j=1}^N | W_j|^2 \big{)} ^{2p} \big{]} \notag \\
&\leq C \epsilon^{-4p} N^{-\beta \times 4p} N^{2p} \mathbb{E}\big{[}  \frac{1}{N} \sum_{j=1}^N | W_j|^{4p}  \big{]} \notag \\
&\leq C \epsilon^{-4p} N^{2p(1 - 2\beta) }. \notag
\end{align}

Therefore, substituting this bound into (\ref{UnionBound2}), 
\begin{align}
\mathbb{P}[ \max_{1 \leq m \leq K(N)} | A_i^N(x_m) |  \geq \epsilon ] &\leq C K(N)    \epsilon^{-4p} N^{2p(1 - 2\beta) } \notag \\
&\leq C ( N^{qn} + 1 ) N^{2p(1 - 2\beta) } \epsilon^{-4p}\notag \\
&\leq C ( N^{(2- \beta) n + 2p (1 - 2 \beta)} + N^{2p(1 - 2\beta) } )\epsilon^{-4p} \notag \\
&\leq C N^{-2}\epsilon^{-4p},
\label{UnionBound3}
\end{align}
where the last line is due to our \emph{a priori} choice of $p> \frac{2 + (2 - \beta)n}{2 (2 \beta -1) }$. Due to the bound (\ref{UnionBound3}),
\begin{eqnarray}
\sum_{N=1}^{\infty} \mathbb{P}[ \max_{1 \leq m \leq K(N)} | A_i^N(x_m) | \geq \epsilon ] \leq C < \infty. 
\end{eqnarray}

Consequently, by the first Borel-Cantelli lemma, $\mathbb{P}[ \max_{1 \leq m \leq K(N)} | A_i^N(x_m) | > \epsilon  \phantom{....} \textrm{i.o.} ] = 0$, which implies that $\mathbb{P}[ \limsup_{N \rightarrow \infty} \max_{1 \leq m \leq K(N)} | A_i^N(x_m) | \leq \epsilon  ] = 1$. Since $\epsilon$ is arbitrary, $\mathbb{P}[ \lim_{N \rightarrow \infty} \max_{1 \leq m \leq K(N)} | A_i^N(x_m) | = 0 ] = 1$. Combining this result with (\ref{ASconvergenceSecondTermY}) yields
\begin{eqnarray}
\mathbb{P}[ \lim_{N \rightarrow \infty} \sup_{x \in \Omega} | A_i^N(x)| = 0 ] = 1. 
\end{eqnarray}
\end{proof}

Define the pre-limit PDE residual for the neural network as $R_t^N(y) = \mathcal{A}_{Q^N(t)} Q^N(y; \theta(t) ) - g(y)$.
\begin{lemma}
\label{PreLimitPDEResidualBoundLemma}
There is a deterministic constant $C_2<\infty$ such that the PDE residual $R_t^N$ is almost surely bounded on a finite time interval $[0,T]$ as $N \rightarrow \infty$:
\begin{eqnarray}
\mathbb{P}[ \limsup_{N \rightarrow \infty} \sup_{t, x \in [0,T] \times \Omega} | R_t^N(x) | \leq C_2 ] = 1,
\label{PreLimitPDEResidualBoundLemmaEqnBound}
\end{eqnarray}
where the constant $C_2$ depends upon $T$.
\end{lemma}
\begin{proof}
The pre-limit PDE residual satisfies
\begin{align}
\frac{\partial R_t^N(y)}{\partial t} &= - \alpha^N N^{- \beta} \phi^N \bigg{(} N^{\beta} \int  ( \mathcal{A}_{Q^N(t ) } Q^N(x; \theta(t) ) - g(x) ) \notag \\
&\quad\times D_{Q^N(t)} [ \nabla_{\theta} Q^N(x; \theta(t) ) ]  \mu(dx) \bigg{)}  \cdot D_{Q^N(t)} [ \nabla_{\theta} Q^N(y; \theta(t) ) ]. 
\label{PrelimitEvolutionC}
\end{align}

Therefore, we can bound the pre-limit PDE residual as
\begin{align}
 |R_t^N(y)| &\leq   |R_0^N(y)| + \alpha N^{\beta - 1} \int_0^t  \bigg{|} N^{\beta} \int_{\Omega}  ( \mathcal{A}_{Q^N(s ) } Q^N(x; \theta(s) ) - g(x) ) \notag \\
 &\quad \times D_{Q^N(s)} [ \nabla_{\theta} Q^N(x; \theta(s) ) ]  \mu(dx) \bigg{|} \cdot | D_{Q^N(s)} [ \nabla_{\theta} Q^N(y; \theta(s) ) ] | ds \notag \\
 &\leq |R_0^N(y)| + \alpha N^{2 \beta - 1} \int_0^t    \int_{\Omega}   |R_s^N(x)|  \times | D_{Q^N(s)} [ \nabla_{\theta} Q^N(x; \theta(s) ) ] |  \mu(dx) \notag \\
 &\quad\cdot | D_{Q^N(s)} [ \nabla_{\theta} Q^N(y; \theta(s) ) ] | ds \notag \\
 &\leq |R_0^N(y)| + \alpha N^{- 1} \int_0^t \int_{\Omega} |R_s^N(x)| \times \sum_{i=1}^N | D_{Q^N(s)}[ \nabla_{\theta_i} ( \eta(x) c^i_s \sigma( w^i_s \cdot x + b^i_s ) ) ] | \notag \\
 &\quad\cdot | D_{Q^N(s)}[ \nabla_{\theta_i} ( \eta(y) c^i_s \sigma( w^i_s \cdot y + b^i_s ) ) ] | \mu(dx) ds. 
\end{align}

Taking a supremum over the RHS and then the LHS yields
\begin{align}
 \sup_{y \in \Omega} |R_t^N(y)| &\leq& \sup_{y \in \Omega} |R_0^N(y)| + C \int_0^t \sup_{x \in \Omega} |R_s^N(x)| \times \frac{1}{N} \sum_{i=1}^N (1 + \norm{w_0^i}_1)^4 ds. \label{PrelimitPDEResidualIntegralBound}
\end{align}

Due to Lemma \ref{MarcinkiewiczZygmundLemma} and since $g$ is uniformly bounded on $\Omega$, there is a constant $C < \infty$ such that $ \mathbb{P}[ \limsup_{N \rightarrow \infty } \sup_{y \in \Omega} |R_0^N(y)| \leq C ] = 1$. By the strong law of large numbers and since the i.i.d. RVs $w_0^i$ have bounded moments, there is a constant $C_1$ such that $\mathbb{P}[ \lim_{N \rightarrow \infty} \frac{1}{N} \sum_{i=1}^N (1 + \norm{w_0^i}_1)^4  = C_1 ] = 1$. Applying Gronwall's inequality to (\ref{PrelimitPDEResidualIntegralBound}) yields that there is a constant $C_2$ such that $\mathbb{P}[ \limsup_{N \rightarrow \infty} \sup_{t, x \in [0,T] \times \Omega} | R_t^N(x) | \leq C_2 ] = 1$. 
\end{proof}

Using the bound from Lemma \ref{PreLimitPDEResidualBoundLemma}, we can now study the second term in equation (\ref{PrelimitEvolutionB}). Recall the definition $L_t^N(y) = \alpha \int_{\Omega} U_t^N(y,x) R^N_t(x) \mu(dx)$. $B_t^N$ can be expressed as the empirical average
\begin{align}
B_t^N(y) &=  \frac{\alpha}{N} \sum_{i=1}^N \phi^N \bigg{(} \int  R_t^N(x) D_{Q^N(t)} [ \nabla_{\theta_i}  ( \eta(x) c^i_t \sigma (w^i_t \cdot x + b^i_t ) ) ] \mu(dx) \bigg{)} \cdot \nabla_{\theta_i} ( \eta(y) c^i_t \sigma(w^i_t \cdot y + b^i_t) ).
\end{align}

$L_t^N$ can be re-arranged in a similar form:
\begin{align}
L_t^N(y) &=  \frac{\alpha}{N} \sum_{i=1}^N \bigg{(}  \int  R_t^N(x) D_{Q^N(t)} [ \nabla_{\theta_i} ( \eta(x) c^i_t \sigma (w^i_t \cdot x + b^i_t ) ) ]  \mu(dx) \bigg{)}  \cdot\nonumber\\
&\qquad\qquad\cdot\nabla_{\theta_i} ( \eta(y) c^i_t \sigma(w^i_t \cdot y + b^i_t) ) \notag
\end{align}

\begin{lemma} \label{ClippingFunctionVanishesAlmostSurely}
The effect of the clipping function $\phi^N$ vanishes almost surely as $N \rightarrow \infty$:
\begin{eqnarray}
\mathbb{P}[ \lim_{N \rightarrow \infty} \sup_{t, y \in [0,T] \times \Omega} | \partial_{y}^i B_t^N(y) - \partial_{y}^i L_t^N(y) |  = 0 ] = 1,
\end{eqnarray}
where $|i| \leq 2$. 
\end{lemma}
\begin{proof}
Define $A_t^{N,k} = \int  R^N_t(x) D_{Q^N(t)} [ \nabla_{\theta_k} ( \eta(x) c^k_t \sigma (w^k_t \cdot x + b^k_t ) ) ]  \mu(dx)$. Then, for $|i| \leq 2$, we can use (\ref{NTKbound1}) to establish the bound
\begin{eqnarray}
| \partial_{y}^i B_t^N(y) - \partial_{y}^i L_t^N(y) | \leq \frac{2 C}{N} \sum_{k=1}^N (1 + \norm{w_0^k}_1 )^2 |A_t^{N,k} | \cdot \mathbf{1}_{|A_t^{N,k}| \geq N^{\gamma}},
\end{eqnarray}
where, as previously, we have used the notation that, for a vector $z$ of length $d$, $|z| = (|z_1|, \ldots, |z_d|)$ and  $\mathbf{1}_{|z| \geq N^{\gamma}} = ( \mathbf{1}_{|z_1| \geq N^{\gamma}}, \ldots,   \mathbf{1}_{ |z_d| \geq N^{\gamma} } )$.

Due again to the bound (\ref{NTKbound1}), we can bound the $|A_t^{N,k} |$ in terms of a polynomial of the initial parameters $w_0^i$:
\begin{eqnarray}
|A_t^{N,k} | \leq C \times \sup_{t, x \in [0,T] \times \Omega} |R^N_t(x) | \times (1 + \norm{ w_0^k }_1 )^2. 
\end{eqnarray}

In Lemma \ref{PreLimitPDEResidualBoundLemma}, it was proven that there exists a deterministic constant $C_2$ such that $\mathbb{P}[ \limsup_{N \rightarrow \infty} \sup_{t, x \in [0,T] \times \Omega} | R_t^N(x) | \leq C_2 ] = 1$. Consequently, there exists a positive constant $C < \infty$ such that
\begin{align}
&\mathbb{P}\left[ \limsup_{N \rightarrow \infty} \sup_{t, y \in [0,T] \times \Omega} | \partial_{y}^i B_t^N(y) - \partial_{y}^i L_t^N(y) |  \leq\nonumber\right.\\
&\left.\qquad\leq\limsup_{N \rightarrow \infty} \frac{C}{N} \sum_{k=1}^N (1 + \norm{ w_0^k }_1 )^4  \mathbf{1}_{C (1 + \norm{ w_0^k }_1 )^2 \geq N^{\gamma}} \right]  = 1.\nonumber
\end{align}

The term on the RHS of the above inequality can be bounded via
\begin{eqnarray}
\frac{C}{N} \sum_{k=1}^N (1 + \norm{ w_0^k }_1 )^4  \mathbf{1}_{C (1 + \norm{ w_0^k }_1 )^2 \geq N^{\gamma}} \leq \frac{C}{N} \sum_{k=1}^N (1 + \norm{ w_0^k }_1 )^4 \mathbf{1}_{C (1 + \norm{ w_0^k }_1 )^4 \geq N^{\gamma}}.
\end{eqnarray}

Define $h(w) = C (1 + \norm{w}_1)^4$ and $H^N = \frac{1}{N} \sum_{k=1}^N h(w_0^k) \mathbf{1}_{h(w_0^k) \geq N^{\gamma} }$. For any $\epsilon > 0$,
\begin{align}
\mathbb{P}[ H^N > \epsilon] &\leq \frac{ \mathbb{E}[ h(w_0^k) \mathbf{1}_{h(w_0^k) \geq N^{\gamma} } ]}{\epsilon} \notag \\
&\leq \frac{ \mathbb{E}[h(w_0^k)  \times \frac{h(w_0^k)^m}{ N^{m \gamma}  } ]}{\epsilon} \notag \\
&\leq \frac{C}{\epsilon} N^{-m \gamma},\notag
\end{align}
where we have used the assumption that $w_0^k$ has bounded moments. We now select $m$ such that $m \gamma > 1$. Then, 
\begin{eqnarray}
\sum_{N=1}^{\infty} \mathbb{P}[ H^N > \epsilon] \leq \frac{C}{\epsilon} < \infty.\notag
\end{eqnarray}

Therefore, by the First Borel-Cantelli Lemma, $\mathbb{P}(H^N > \epsilon \phantom{...} \textrm{i.o.} ) = 0$. Consequently, for any $\epsilon > 0$, $\mathbb{P}(\limsup_{N \rightarrow \infty} H^N > \epsilon ) = 0$. Since $H^N$ is non-negative, this implies that $\mathbb{P}(\lim_{N \rightarrow \infty} H^N = 0 ) = 1$. 

We can therefore conclude that
\begin{eqnarray}
\mathbb{P}[ \lim_{N \rightarrow \infty} \sup_{t, y \in [0,T] \times \Omega} | \partial_y^i B_t^N(y) - \partial_y^i L_t^N(y) |  = 0 ] = 1. \notag
\end{eqnarray}
\end{proof}

We next bound the difference between $U_t^N$ and $U_{Q(t)}$.
\begin{lemma} \label{UniformInTimeBound2Lemma}
For $|i| \leq 2$ and $t \in [0,T]$,
\begin{align} \label{UniformInTimeBound2}
\sup_{x,y \in \Omega^2} | \partial^i_{y} [ U_t^N(y,x) - U_{Q(t)}(y,x) ] | &\leq K^{N,1}(t) + K^{N,2} \sup_{x \in \Omega} (  | Q^N(x; \theta(t)) - Q(t,x) | \notag \\
&+ \sum_{m=1}^n | Q^N_{x_m}(x; \theta(t)) - Q_{x_m}(t,x) | ),  
\end{align}
where $K^{N,2} =  \frac{C_1}{N} \sum_{i=1}^N (1 +  \norm{w^i_0}_1 )^4$. $\sup_{t \in [0,T] } K^{N,1}(t) \overset{a.s.} \rightarrow 0$ and $\lim_{N \rightarrow \infty}  K^{N,2} \overset{a.s.} = C$, where $C$ is a finite constant.
\end{lemma}

\begin{proof}
By direct computation we have
\begin{align}
& \partial_{y}^i [ U_t^N(y,x) - U_{Q(t)}(y,x) ] \notag \\
&= \bigg{\langle}  \partial_{y}^i [ \nabla_{\theta} ( \eta(y) c \sigma (w y + b) ) ] \cdot \bigg{(}   D_{Q^N(t)}[ \nabla_{\theta} ( \eta(x) c \sigma (w x + b) )  ]  -  D_{Q(t)} [ \nabla_{\theta} ( \eta(x) c \sigma (w x + b) )  ] \bigg{)}, \mu_t^N \bigg{\rangle} \notag \\
&\quad+ \bigg{\langle} \partial_y^i [ \nabla_{\theta} ( \eta(y) c \sigma (w y + b) ) ]  \cdot  D_{Q(t)}[ \nabla_{\theta} ( \eta(x) c \sigma (w x + b) )  ], \mu_t^N - \mu_0^N \bigg{\rangle} \notag \\
&\quad+ \bigg{\langle} \partial_y^i  [ \nabla_{\theta}  ( \eta(y) c \sigma (w y + b) ) ] \cdot D_{Q(t)}[ \nabla_{\theta}  ( \eta(x) c \sigma (w x + b) ) ], \mu_0^N - \mu_0 \bigg{\rangle}.
\label{UdiffDecomposition}
\end{align}

The second term on the RHS of equation (\ref{UdiffDecomposition}) converges to zero uniformly in $t,x,y \in [0,T] \times \Omega \times \Omega$ at a rate $\mathcal{O}(N^{\beta -1 + \gamma})$ due to (\ref{NTKbound0}). 

Using a Taylor expansion as well as the bounds for trained parameters (\ref{NTKbound0}) and (\ref{TrainedParameterCBound}), the first term on the RHS of (\ref{UdiffDecomposition}) can be bounded as a function of $| Q^N(x; \theta(t)) - Q(t,x) | + \sum_{m=1}^n  | Q^N_{x_m}(x; \theta(t)) - Q_{x_m}(t,x) | $:
\begin{align}
&  | \bigg{\langle} \partial_y^i [ \nabla_{\theta} ( \eta(y) c \sigma (w y + b) ) ] \cdot \bigg{(}  D_{Q^N(t)}[ \nabla_{\theta} ( \eta(x) c \sigma (w x + b) ) ]  - D_{Q(t)}[ \nabla_{\theta} ( \eta(x) c \sigma (w x + b) )  ] \bigg{)}, \mu_t^N \bigg{\rangle}  | \notag \\
&\leq  C_1 \big{(}  | Q^N(x; \theta(t)) - Q(t,x) | \nonumber\\
&\quad + \sum_{m=1}^n | Q^N_{x_m}(x; \theta(t)) - Q_{x_m}(t,x) | \big{)} \times  \frac{1}{N} \sum_{i'=1}^N ( 1 + \norm{w^{i'}_0}_1 )^4. \notag 
\end{align}

Recall that $w^i_0$ are i.i.d. random variables. By the strong law of large numbers,
\begin{eqnarray}
\frac{1}{N} \sum_{i=1}^N ( 1 + \norm{w^i_0}_1 )^4 \overset{a.s.} \rightarrow \mathbb{E} [( 1 + \norm{w^i_0}_1)^4 ] = C_3,\notag
\end{eqnarray}
as $N \rightarrow \infty$. 

The third term in (\ref{UdiffDecomposition}) can be bounded by: 
\begin{eqnarray}
&\phantom{.}& \sup_{t,x,y \in [0,T] \times \Omega^2} | \frac{1}{N} \sum_{i=1}^N [q(c_0^i, w_0^i, b_0^i ,t, x, y)  - \mathbb{E}_{\mu_0}[ q(c_0^i, w_0^i, b_0^i ,t, x, y)] ] |,
\label{ThirdLineExpanded}
\end{eqnarray}
where $q(c,w,b,t,x,y) = \partial_{y}^i  [ \nabla_{\theta} ( \eta(y) c \sigma (w \cdot y + b) ) ] \cdot   D_{Q(t)}[ \nabla_{\theta} ( \eta(x) c \sigma (w \cdot x + b) ) ]$. For a $\delta > 0$, construct a finite sub-cover $U_k = \{ (t,x,y): \norm{ (t,x,y) - (t_k,x_k,y_k) }_1 \leq \delta \}$ with $k = 0, 1, \ldots, K$ for the compact set $[0,T] \times \Omega^2$. Define $A^N(t,x,y) = \frac{1}{N} \sum_{i=1}^N [q(c_0^i, w_0^i, b_0^i ,t, x, y)  - \mathbb{E}_{\mu_0}[ q(c_0^i, w_0^i, b_0^i ,t, x, y)] ]$. Note that the quantity  $\{ A^N(t_k, x_k, y_k) \}_{ k = 0, 1, \ldots, K}$ converges almost surely to zero. We begin our calculations by considering the bound 
\begin{align}
&|A^N(t,x,y) - A^N(t_k, x_k, y_k)| \leq \frac{1}{N} \sum_{i=1}^N \bigg{(} |q(c_0^i, w_0^i, b_0^i ,t, x, y) - q(c_0^i, w_0^i, b_0^i ,t_k, x_k, y_k) | \notag \\
&\qquad+ \mathbb{E}_{\mu_0}[ |q(c_0^i, w_0^i, b_0^i ,t, x, y) - q(c_0^i, w_0^i, b_0^i ,t_k, x_k, y_k) | ] \bigg{)}. 
\label{CalculationA1}
\end{align}

Using a Taylor expansion, the fact that $\sigma$ has at least three bounded derivatives, the bound on $c_0^i$, and the uniform bounds on the time derivatives of $Q$ and $Q_x$ from Lemma \ref{QTimederivativeUniformBound}:
\begin{align}
| q(c_0^i, w_0^i, b_0^i ,t, x, y) - q(c_0^i, w_0^i, b_0^i ,t_k, x_k, y_k) | &\leq C ( 1 + \norm{w_0^i}_1 )^5 \times ( | y - y_k |  + | x - x_k | \notag \\
&\quad+ |t - t_k | ).
\label{CalculationA2}
\end{align}

Combining (\ref{CalculationA1}) with (\ref{CalculationA2}) and setting $f(w) = (1 + \norm{w}_1)^5$,
\begin{align}
\sup_{t,x,y \in U_k} | A^N(t,x,y) - A^N(t_k, x_k, y_k) | &\leq \frac{C_1 \delta}{N} \sum_{j=1}^N f(w^j_0) + C_2 \delta. \notag
\end{align}
By the strong law of large numbers, the RHS converges almost surely to $C_1 \delta \mathbb{E}[ f(w^j_0) ] + C_2 \delta$. Consequently, 
\begin{eqnarray}
\limsup_{N \rightarrow \infty} \sup_{t,x,y \in [0,T] \times \Omega^2} | A^N(t,x,y)| \overset{a.s.} \leq C \delta.
\end{eqnarray}

Since $\delta$ was arbitrary, the above inequality implies that 
\[
\lim_{N \rightarrow \infty} \sup_{t,x,y \in [0,T] \times \Omega^2} | A^N(t,x,y)| \overset{a.s.} = 0.
\]

Consequently, we have calculated the following bound
\begin{align} \label{UniformInTimeBound22}
\sup_{x,y \in  \Omega^2} | \partial_{y}^i [ U_t^N(y,x) - U_t(y,x) ] | &\leq K^{N,1}(t) + K^{N,2} \big{(} | Q^N(x; \theta(t)) - Q(t,x) | \notag \\
&+  \sum_{m=1}^n | Q^N_{x_m}(x; \theta(t)) - Q_{x_m}(t,x) | \big{)},
\end{align}
where  $\sup_{t \in [0,T] } K^{N,1}(t) \overset{a.s.} \rightarrow 0$, $K^{N,2} =  \frac{C_1}{N} \sum_{i=1}^N ( 1 + \norm{w^i_0}_1 )^4$, and $\lim_{N \rightarrow \infty}  K^{N,2} \overset{a.s.} = C$, where $C$
is some finite constant. 
\end{proof}

\begin{theorem} \label{ConvergenceNtoLimitPDE}
$\partial_{x}^i Q^N(t,x)$ converges almost surely to $\partial_{x}^i  Q(t,x)$ on $t \in [0,T]$ as $N \rightarrow \infty$ for $|i| \leq 2$:
\begin{eqnarray}
\sup_{t \in [0,T] } \sup_{x \in \Omega} | \partial_{x}^i Q^N(t,x) - \partial_{x}^i Q(t,x) | \overset{a.s.} \rightarrow 0,\notag
\end{eqnarray}
where $T < \infty$ is arbitrary. As a direct consequence, as $N \rightarrow \infty$,
\begin{eqnarray}
\sup_{t \in [0,T] } \norm{ Q^N(t) -  Q(t) }_{H^2} \overset{a.s.} \rightarrow 0.\notag
\end{eqnarray}

\end{theorem}
 
\begin{proof}
For sufficiently large $N$, the difference $V^N(t,x) =  Q(t,x) - Q^N(t,x)$ satisfies
\begin{align}
\frac{d V^N}{d t}(t,x) &= - \alpha  \int_{\Omega} \bigg{(}  U_{Q(t)}(x,y)  \big{(} \mathcal{A}_{Q(t) } Q(t, y ) - g(y) \big{)} \notag \\
&- U_t^N(x,y) \big{(} \mathcal{A}_{Q^N(t) } Q^N(y; \theta(t) ) - g(y) \big{)} \bigg{)}  \mu(dy)  + (L_t^N(x) - B_t^N(x) ) \notag \\
&= - \alpha  \int_{\Omega} (U_{Q(t)}(x,y) - U_t^N(x,y)  )  \big{(} \mathcal{A}_{Q(t ) } Q(t, y ) - g(y) \big{)}  \mu(dy) \notag \\
&\quad- \alpha  \int_{\Omega}  U_t^N(x,y)   \big{(} \mathcal{A}_{Q(t) } Q(t, y ) - g(y) - ( \mathcal{A}_{Q^N(t ) } Q^N(t, y ) - g(y) )  \big{)}  \mu(dy) \notag\\
&\quad + (L_t^N(x) - B_t^N(x) ). \notag 
\end{align}

Define $E^N(t,x) =  \sum_{|i| \leq 2} | \partial_{x}^i[  Q(t,x) - Q^N(t,x) ]| $. Differentiating on both sides of the equation, integrating over time, and using a Taylor expansion yields
\begin{align}
&| \partial_{x}^i[  Q(t,x) - Q^N(t,x) ]| \leq |\partial_{x}^i [  Q(0,x) - Q^N(0,x) ]| \notag \\
&\quad+ \alpha \int_0^t  \int_{\Omega} | \partial_{x}^i (U_{Q(s)}(x,y) - U_s^N(x,y)  ) |  \times | \mathcal{A}_{Q(s ) } Q(s, y ) - g(y) |  \mu(dy) ds \notag \\
&\quad+ \alpha  \int_0^t \int_{\Omega} | \partial_{x}^i U_s^N(x,y)  | \times \big{|}  \mathcal{A}_{Q(s) } Q(s, y ) - \mathcal{A}_{Q^N(s ) } Q^N(s, y )  \big{|}  \mu(dy) ds \notag \\
&\quad+ \int_0^t | \partial_{x}^i L_s^N(x) - \partial_{x}^i  B_s^N(x)  | ds.\notag
\end{align}

Summing over $|i| \leq 2$ and using the bound from equation (\ref{UniformVBound0}):
\begin{align}
E^N(t,x) &\leq E^N(0,x) + \sum_{|i| \leq 2} \bigg{(} C \int_0^t  \int_{\Omega}  | \partial_{x}^i (U_{Q(s)}(x,y) - U_s^N(x,y)  ) | \mu(dy) ds \notag \\
&+ C \int_0^t \int_{\Omega} | \partial_x^i U_s^N(x,y)  | \times E^N(s,y) \mu(dy) ds \notag \\
&+ \int_0^t | \partial_{x}^i L_s^N(x) - \partial_{x}^i  B_s^N(x)  | ds \bigg{)}. \notag 
\end{align}

Taking the supremum over the RHS:
\begin{align}
E^N(t,x) &\leq \sup_{x \in \Omega} E^N(0,x) \notag\\
&\quad+ \sum_{|i| \leq 2} \bigg{(} C \int_0^t  \int_{\Omega}   \sup_{x,y \in \Omega^2} | \partial_{x}^i (U_{Q(s)}(x,y) - U_s^N(x,y)  ) | \mu(dy) ds \notag \\
&\quad+ C \int_0^t \int_{\Omega} \sup_{x,y \in \Omega^2}| \partial_{x}^i U_s^N(x,y)  | \times \sup_{y \in \Omega} E^N(s,y) \mu(dy) ds \notag \\
&\quad+  C \sup_{t,x \in [0,T] \times \Omega} | \partial_{x}^i L_t^N(x) - \partial_{x}^i  B_t^N(x)  | \bigg{)}.
\end{align}

Using the bound (\ref{UniformInTimeBound2}), the above equation can be simplified to give:
\begin{align}
E^N(t,x) &\leq \sup_{x \in \Omega} E^N(0,x) + C \sup_{t \in [0,T]} K^{N,1}(t) \notag \\
&\quad+ C \int_0^t  ( \sum_{|i| \leq 2}  \sup_{x,y \in \Omega^2} |  \partial_{x}^i U_s^N(x,y)  | +K^{N,2} ) \times \sup_{y \in \Omega} E^N(s,y) ds \notag \\
&\quad+ C \sum_{|i| \leq 2} \sup_{t,x \in [0,T] \times \Omega} | \partial_{x}^i L_t^N(x) - \partial_{x}^i  B_t^N(x)  |, \notag 
\end{align}
where $K^{N,1}(t)$ and $K^{N,2}$ were defined in Lemma \ref{UniformInTimeBound2Lemma}. 

Since the above equation holds for every $x \in \Omega$, we can take the supremum of the LHS:
\begin{align}
\sup_{x \in \Omega} E^N(t,x) &\leq  \sup_{x \in \Omega} E^N(0,x) + C \sup_{t \in [0,T]} K^{N,1}(t) \notag \\
&+ C \int_0^t  \left( \sum_{|i| \leq 2}  \sup_{s, x,y \in [0,T] \times \Omega^2} |  \partial_{x}^i U_s^N(x,y)  | +K^{N,2} \right) \times \sup_{y \in \Omega} E^N(s,y) ds \notag \\
&+ K^{N,4},
\label{supremumEbound}
\end{align}
where $K^{N,4} = C \sum_{|i| \leq 2} \sup_{t,x \in [0,T] \times \Omega} | \partial_{x}^i L_t^N(x) - \partial_{x}^i  B_t^N(x)  |$. The term $K^{N,4}$ was previously studied in Lemma \ref{ClippingFunctionVanishesAlmostSurely} and converges almost surely to zero as $N \rightarrow \infty$. Due to the bound (\ref{NTKbound0}) and the assumptions on the initial parameters, we can prove that $K^{N,3} = \sum_{|i| \leq 2}  \sup_{t, x,y \in [0,T] \times \Omega^2} |  \partial_{x}^i U_t^N(x,y)  | $ is almost surely bounded due to the strong law of large numbers. That is, there exists a finite constant $C_3$ such that  
\begin{eqnarray}
\mathbb{P}[\limsup_{N \rightarrow \infty } \sum_{|i| \leq 2}  \sup_{t, x,y \in [0,T] \times \Omega^2}   |  \partial_{x}^i U_t^N(x,y)  | \leq C_3 ] = 1. \notag
\end{eqnarray}

Recall that 
\begin{eqnarray}
E^N(0,x) = \sum_{|i| \leq 2} | \partial_{x}^i [ \eta(x) N^{-\beta} \sum_{j=1}^N c^j_0 \sigma(w^j_0 \cdot x + b^j_0) ] |.\notag
\end{eqnarray}
Therefore, due to Lemma \ref{MarcinkiewiczZygmundLemma}, $\sup_{x \in \Omega} E^N(0,x) \overset{a.s.} \rightarrow 0$ as $N \rightarrow \infty$.

Applying Gronwall's inequality to (\ref{supremumEbound}) yields
\begin{align}
\sup_{x \in \Omega} E^N(t,x) &\leq  \sup_{x \in \Omega} E^N(0,x) + K^{N,4} + C \sup_{t \in [0,T] } K^{N,1}(t) \notag \\
&\quad+ t \bigg{(} \sup_{x \in \Omega} E^N(0,x) + K^{N,4} + C \sup_{t \in [0,T] } K^{N,1}(t) \bigg{)} \notag \\
&\quad\times (K^{N,2} + K^{N,3} ) 
\times \exp \bigg{(} (K^{N,2} + K^{N,3} ) t \bigg{)}. \notag
\end{align}

Consequently, for deterministic finite constants $C_2$ and $C_3$,
\begin{align}
\sup_{x \in \Omega} E^N(0,x) &\overset{a.s.} \rightarrow 0,\notag\\
K^{N,4} &\overset{a.s.} \rightarrow 0,\notag\\ \sup_{t \in [0,T]} K^{N,1}(t) &\overset{a.s.} \rightarrow 0\notag\\ 
\mathbb{P}[ \lim_{N \rightarrow \infty} K^{N,2}  = C_2] &= 1, \notag\\ 
\mathbb{P}[ \limsup_{N \rightarrow \infty} K^{N,3} \leq C_3 ] &= 1,
\notag
\end{align} and therefore
\begin{eqnarray}
\sup_{t \in [0,T] } \sup_{x \in \Omega} E^N(t,x) \overset{a.s.} \rightarrow 0,\notag
\end{eqnarray}
as $N \rightarrow \infty$. This concludes the proof of the theorem.

\end{proof}

\section{Global convergence of finite neural network to PDE solution for finite $N$ as $t \rightarrow \infty$}\label{S:GlobalConvergenceFinitN_large_t}

The convergence of the finite neural network $Q^N(t,x)$ to the solution $Q(t,x)$ of the limit PDE in Theorem \ref{ConvergenceNtoLimitPDE} is not uniform in time $t \geq 0$. Due to this lack of uniform convergence, it is often difficult in typical mean-field analyses to establish uniform-in-time error bounds on the pre-limit process $Q^N(t,x)$. However, in the following analysis, we are able to prove uniform-in-time error bounds on the objective function $J^N_t = J^N(\theta(t))$ for the PDE residual of the finite neural network by leveraging the monotonicity of $J^N_t$.

The PDE residual is monotonically decreasing since
\begin{align}
J^N_t &= J^N_s - \alpha^N \int_s^t \nabla_{\theta} J^N(\theta(\tau)) \cdot G^N(\theta(\tau)) d \tau \notag \\
&= J^N_s - \alpha^N \int_s^t N^{-2 \beta} \sum_{i=1}^N N^{\beta} \nabla_{\theta_i} J^N(\theta(\tau)) \cdot \phi^N \big{(} N^{\beta} \nabla_{\theta_i} J^N(\theta(\tau)) \big{)} d \tau \notag \\
&\leq J^N_s - \alpha^N \int_s^t N^{-2 \beta} \sum_{i=1}^N \phi^N \big{(} N^{\beta} \nabla_{\theta_i} J^N(\theta(\tau)) \big{)} \cdot \phi^N \big{(} N^{\beta} \nabla_{\theta_i} J^N(\theta(\tau)) \big{)} d \tau \notag \\
&\leq J^N_s.\notag 
\end{align}

Then, for any $0 \leq s \leq t$,
\begin{eqnarray}
\mathbb{E}[ J_t^N] \leq \mathbb{E}[ J_s^N]. \notag
\end{eqnarray}
Due to Assumption \ref{NeuralNetworkAssumptions}, $\mathbb{E}[ J_0^N] \leq C < \infty$. Since $J_t^N \geq 0$, $\mathbb{E}[ J_t^N] $ converges to a limit as $t \rightarrow \infty$ due to the monotone convergence theorem. 

Theorem \ref{ConvergenceNtoLimitPDE} directly implies that, for each $t \geq 0$, $J_t^N$ converges almost surely to $J(t)$ as $N \rightarrow \infty$. Due to Assumption \ref{NeuralNetworkAssumptions}, $\mathbb{E}[ (J_0^N)^2] \leq C < \infty$. Then, due to monotonicity, $\mathbb{E}[ (J_t^N)^2] \leq C$, which implies that $J_t^N$ is uniformly integrable. Consequently, for each $t \geq 0$, $\mathbb{E}[ J_t^N]$ converges to $J(t)$. Theorem \ref{MainTheoremPDEResidualGlobalConvergenceStrong} has proven that $\lim_{t \rightarrow \infty} J(t) = 0$. Therefore, for any $\epsilon > 0$, there exists a $t \geq 0$ such that $J(t) \leq \frac{\epsilon}{2}$ and an $N_0$ such that $| \mathbb{E}[ J_t^N] - J(t) |  \leq \frac{\epsilon}{2}$ for any $N \geq N_0$. Consequently, we have that
\begin{align}
 \mathbb{E}[ J_t^N]  &= |  \mathbb{E}[ J_t^N] - J(t) + J(t) | \notag \\
 &\leq |  \mathbb{E}[ J_t^N] - J(t) | + |J(t) | \notag \\
 &\leq \epsilon. 
\end{align}

Furthermore, due to monotonicity, 
\begin{eqnarray}
 \mathbb{E}[ J_s^N] \leq \epsilon,\notag
\end{eqnarray}
for all $s \geq t$ and $N \geq N_0$. That is, we can prove a uniform error bound for the neural network with a finite number of hidden units. 
\begin{theorem} \label{FiniteNetworkPDEresidualUniformErrorBound}
For any $\epsilon > 0$, there exists a $t \geq 0$ and an $N_0$ such that
\begin{eqnarray}
 \mathbb{E}[ J_s^N] \leq \epsilon,\notag
\end{eqnarray}
for all $s \geq t$ and $N \geq N_0$. Consequently, we obtain that
\begin{eqnarray}
\lim_{N \rightarrow \infty} \lim_{s \rightarrow \infty}  \mathbb{E}[ J_s^N]  = 0. \notag
\end{eqnarray}

\end{theorem}

\section*{Acknowledgements}
KS and JS were partially supported by the project ``DMS-EPSRC: Asymptotic Analysis of Online Training Algorithms in Machine Learning: Recurrent, Graphical, and Deep Neural Networks" (NSF DMS-2311500).

\end{document}